\documentclass[letterpaper,12pt]{article}

\usepackage{fullpage}

\usepackage[utf8]{inputenc} % allow utf-8 input
\usepackage[T1]{fontenc}    % use 8-bit T1 fonts
\usepackage{hyperref}       % hyperlinks
\usepackage{url}            % simple URL typesetting
\usepackage{booktabs}       % professional-quality tables
\usepackage{nicefrac}       % compact symbols for 1/2, etc.
\usepackage{microtype}      % microtypography
\usepackage{xcolor}         % colors
\usepackage[authoryear, round]{natbib}

\usepackage{times}

\usepackage{amsthm,amsmath,amsfonts,amssymb,amscd}
\usepackage{mathrsfs}
\usepackage[shortlabels]{enumitem}
\usepackage{bm,bbm}

\usepackage{thmtools}
\usepackage{thm-restate}

\usepackage{mh}
\usepackage{defns_multisite}

\title{Learning from Biased and Costly Data Sources: Minimax-optimal Data Collection under a Budget}

\author{%
	Michael O. Harding\\
	Department of Statistics\\
	University of Wisconsin-Madison\\
	moharding@wisc.edu
	\and
	Vikas Singh\\
	Department of Biostatistics\\
	University of Wisconsin-Madison\\
	vsingh@biostat.wisc.edu
	\and
	Kirthevasan Kandasamy\\
	Department of Computer Sciences\\
	University of Wisconsin-Madison\\
	kandasamy@cs.wisc.edu
}

\date{}

\begin{document}

\maketitle

\newcommand{\insertsimfigure}{
\begin{figure}[h]
	\hspace*{\fill}
	\begin{minipage}[t]{0.49\linewidth}
    Setting 1
		\centering
		\vspace{0pt}
		\includegraphics[width=0.95\linewidth]{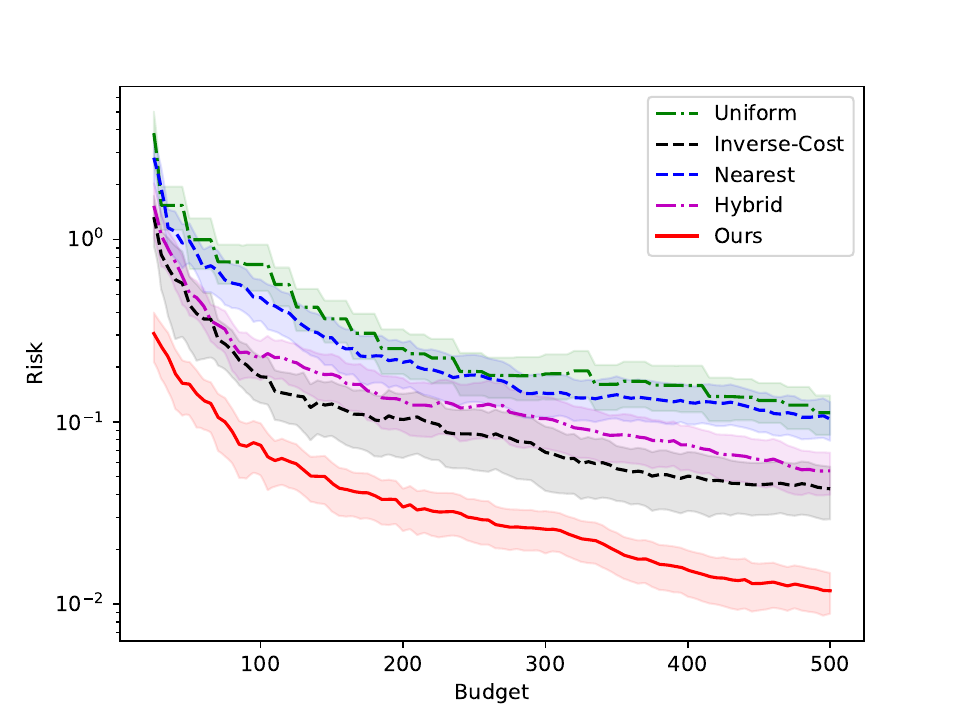}
	\end{minipage}
	\hfill
	\begin{minipage}[t]{0.49\linewidth}
    Setting 2
		\centering
		\vspace{0pt}
		\includegraphics[width=0.95\linewidth]{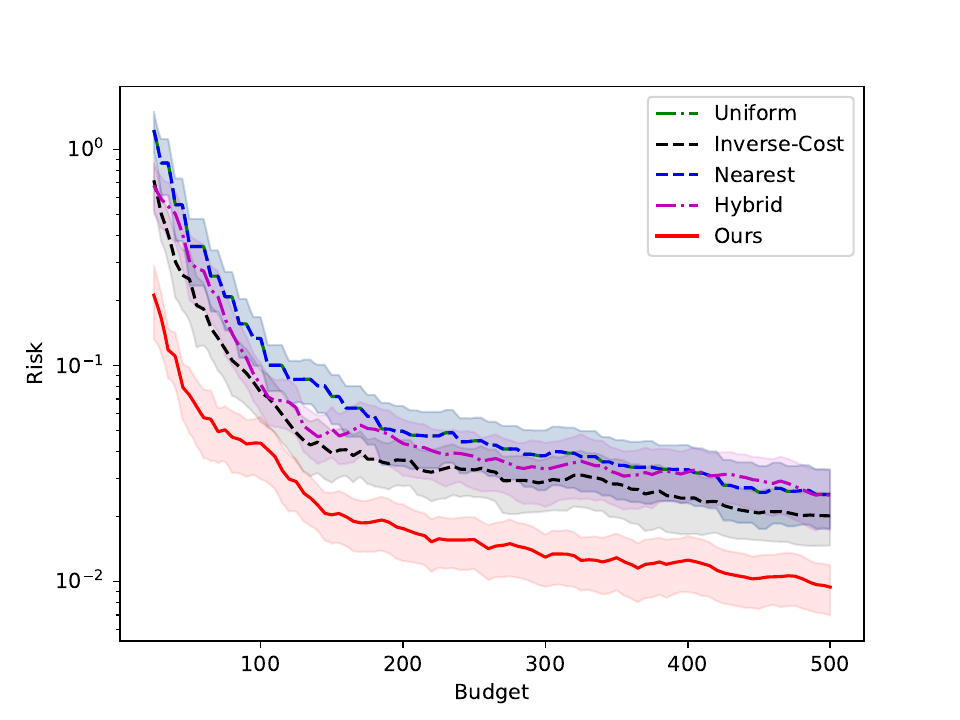}
	\end{minipage}
	\hspace*{\fill}
	
	\hspace*{\fill}
	\begin{minipage}[t]{0.49\linewidth}
		\centering
		\vspace{0pt}
		\includegraphics[width=0.95\linewidth]{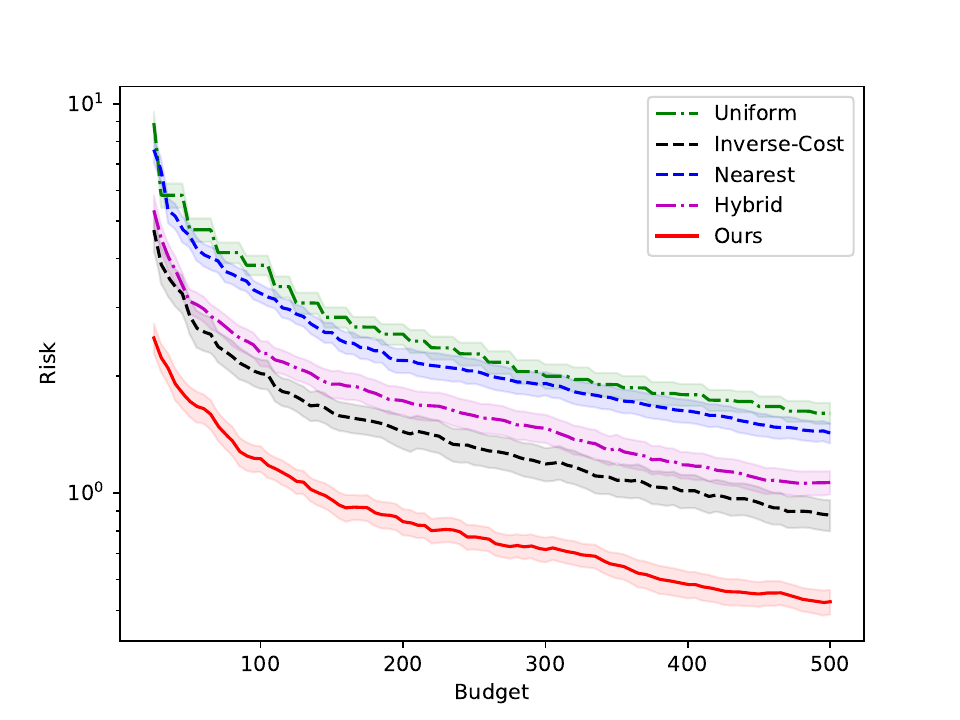}
	\end{minipage}
	\hfill
	\begin{minipage}[t]{0.49\linewidth}
		\centering
		\vspace{0pt}
		\includegraphics[width=0.95\linewidth]{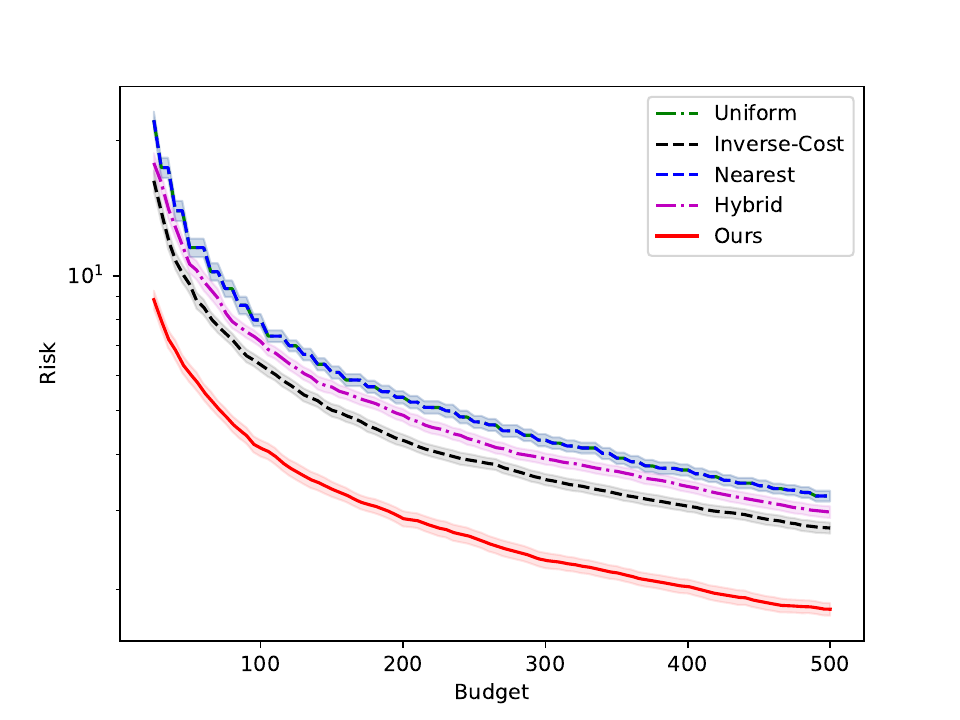}
	\end{minipage}
	\hspace*{\fill}

    \hspace*{\fill}
	\begin{minipage}[t]{0.49\linewidth}
		\centering
		\vspace{0pt}
		\includegraphics[width=0.95\linewidth]{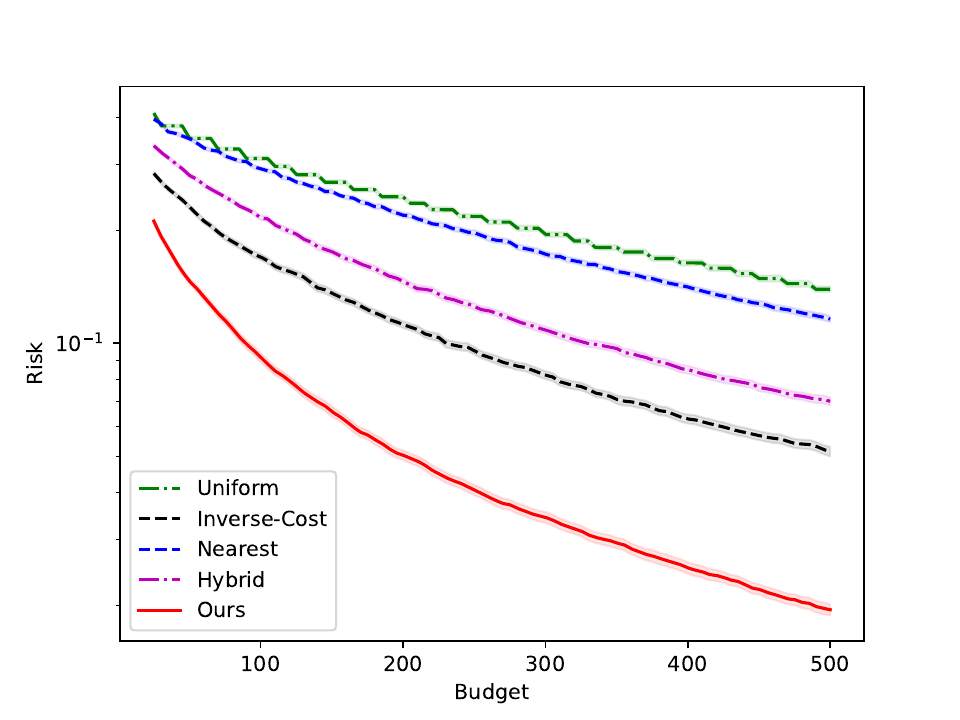}
	\end{minipage}
	\hfill
	\begin{minipage}[t]{0.49\linewidth}
		\centering
		\vspace{0pt}
		\includegraphics[width=0.95\linewidth]{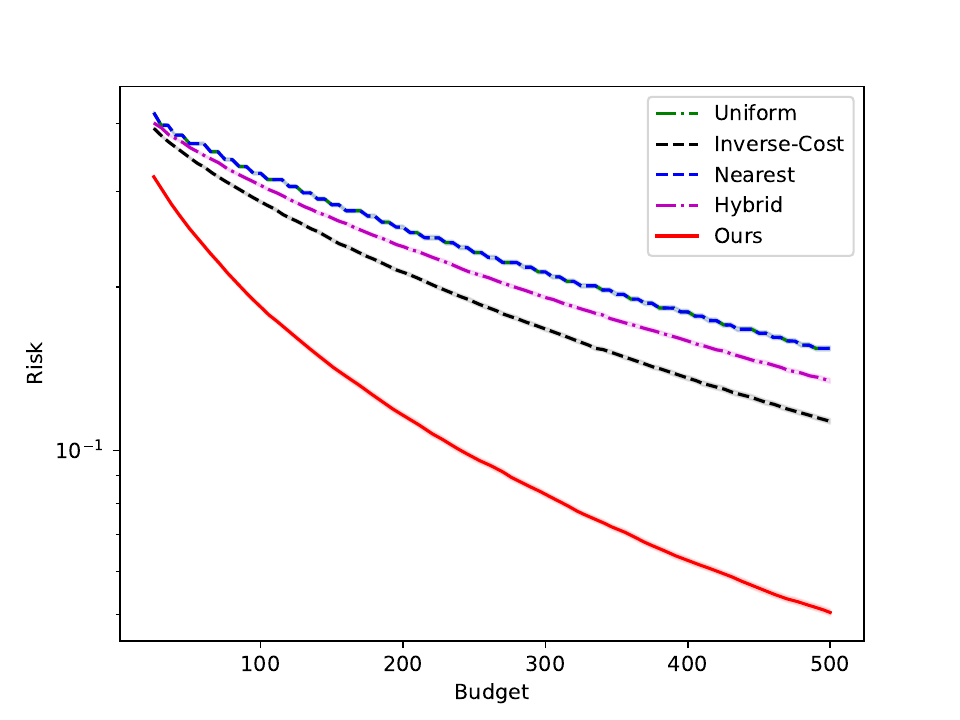}
	\end{minipage}
	\hspace*{\fill}
	\caption{Estimated risk based on 100 simulations in each setting. Error regions represent empirical average $\pm$ 2 SE. Row 1: Population mean under $\unif\,$. Row 2: vector of group means. Row 3: Binary classification under $\unif$.}
    \label{fig:unif}
\end{figure}
}

\newcommand{\insertsimfigureincr}{
\begin{figure}[h]
	\hspace*{\fill}
	\begin{minipage}[t]{0.49\linewidth}
    Setting 1
		\centering
		\vspace{0pt}
		\includegraphics[width=0.95\linewidth]{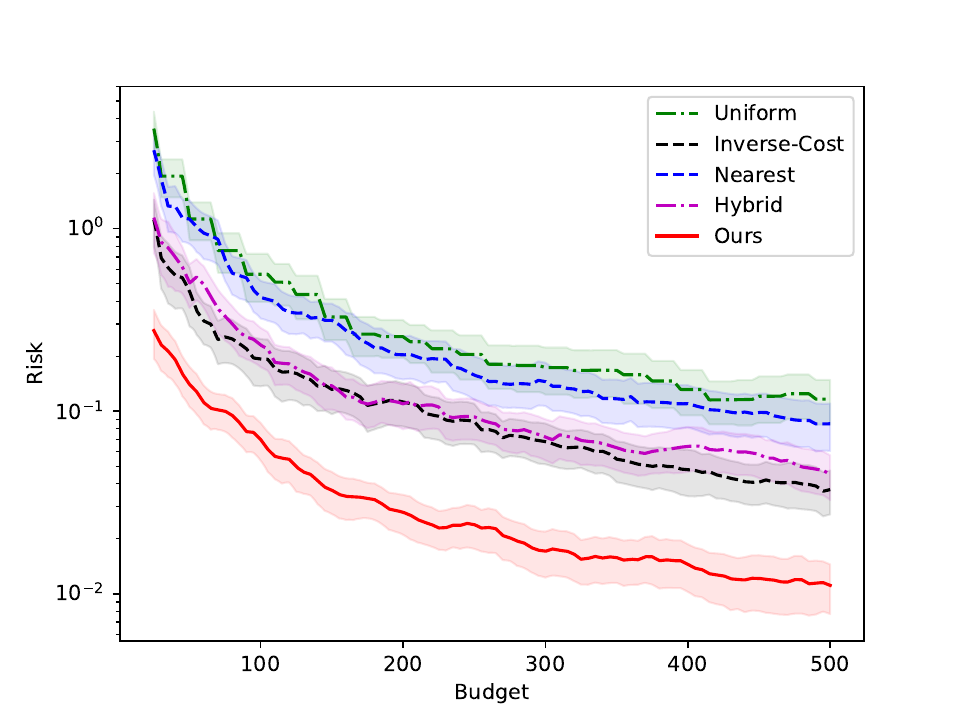}
	\end{minipage}
	\hfill
	\begin{minipage}[t]{0.49\linewidth}
    Setting 2
		\centering
		\vspace{0pt}
		\includegraphics[width=0.95\linewidth]{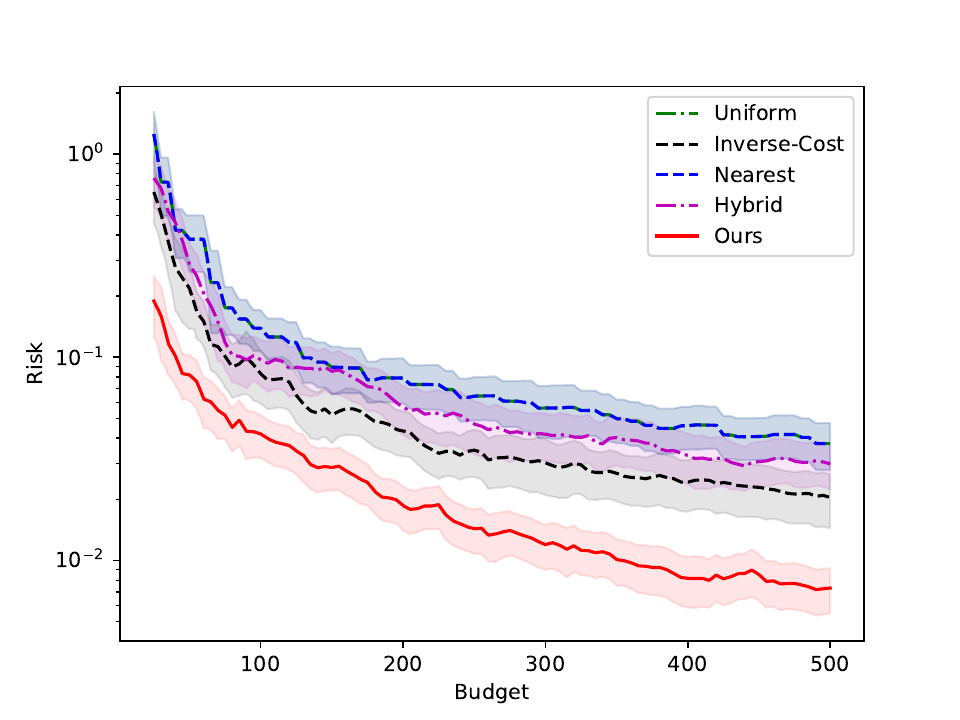}
	\end{minipage}
	\hspace*{\fill}
	
	\hspace*{\fill}
	\begin{minipage}[t]{0.49\linewidth}
		\centering
		\vspace{0pt}
		\includegraphics[width=0.95\linewidth]{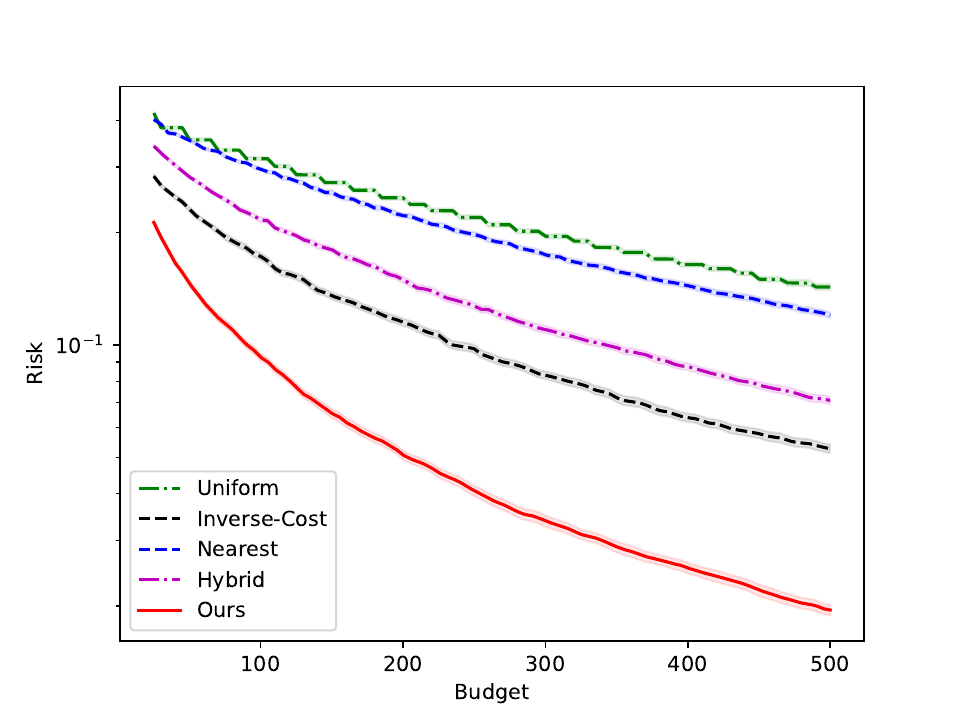}
	\end{minipage}
	\hfill
	\begin{minipage}[t]{0.49\linewidth}
		\centering
		\vspace{0pt}
		\includegraphics[width=0.95\linewidth]{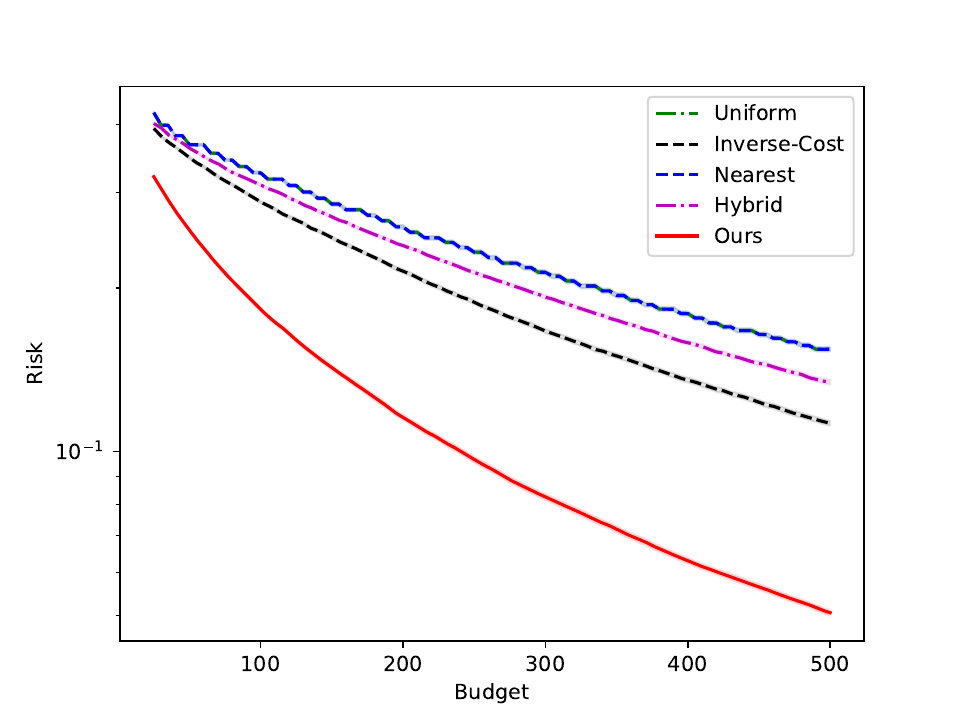}
	\end{minipage}
	\hspace*{\fill}
    \caption{Estimated risk based on 100 simulations in each setting. Error regions represent empirical average $\pm$ 2 SE. Row 1: Population mean. Row 2: Binary classification.}
    \label{fig:incr}
\end{figure}
}

\newcommand{\insertsimfigurepyrm}{
\begin{figure}[h]
	\hspace*{\fill}
	\begin{minipage}[t]{0.49\linewidth}
    Setting 1
		\centering
		\vspace{0pt}
		\includegraphics[width=0.95\linewidth]{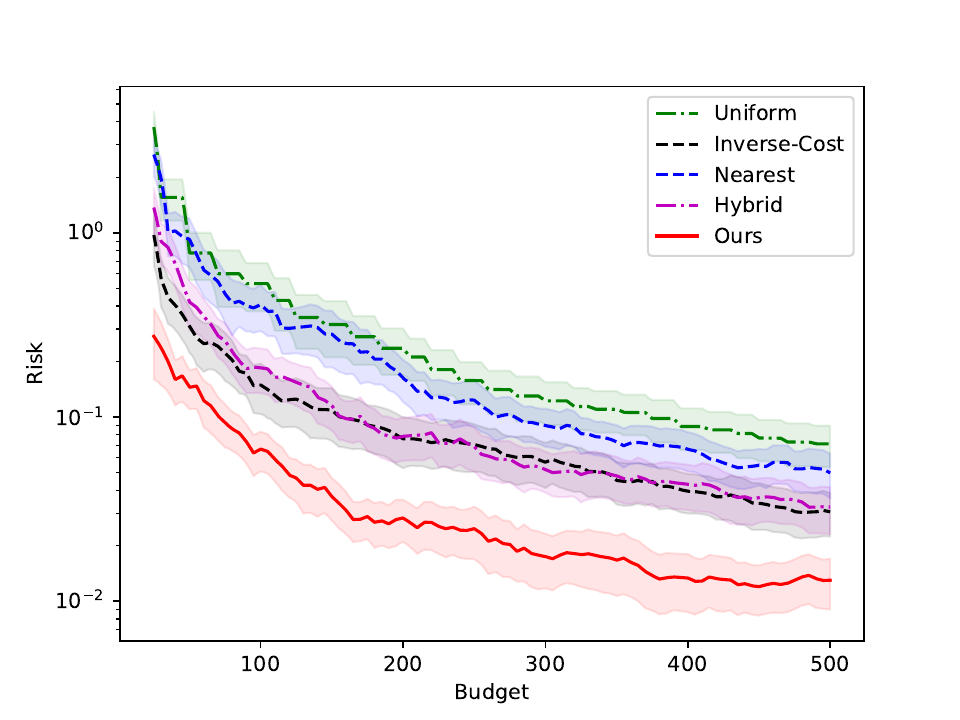}
	\end{minipage}
	\hfill
	\begin{minipage}[t]{0.49\linewidth}
    Setting 2
		\centering
		\vspace{0pt}
		\includegraphics[width=0.95\linewidth]{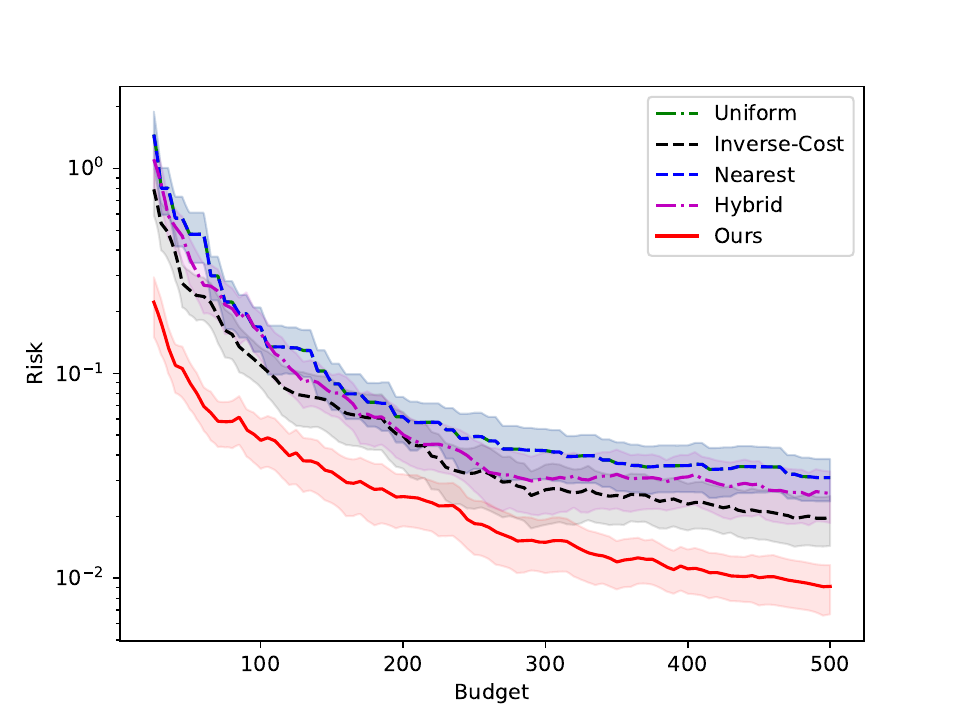}
	\end{minipage}
	\hspace*{\fill}
	
	\hspace*{\fill}
	\begin{minipage}[t]{0.49\linewidth}
		\centering
		\vspace{0pt}
		\includegraphics[width=0.95\linewidth]{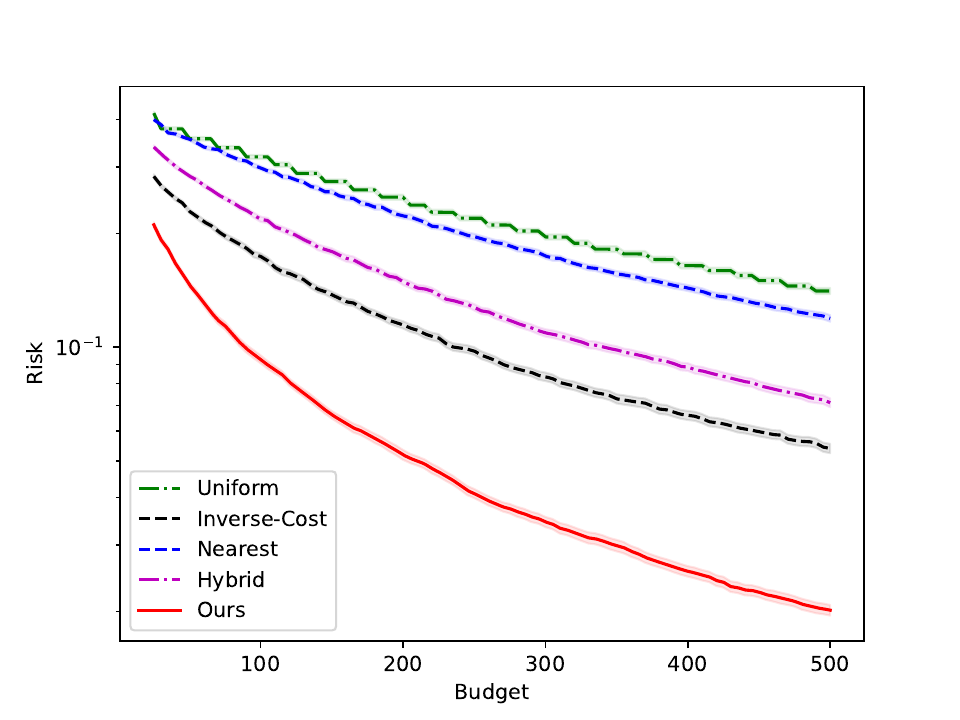}
	\end{minipage}
	\hfill
	\begin{minipage}[t]{0.49\linewidth}
		\centering
		\vspace{0pt}
		\includegraphics[width=0.95\linewidth]{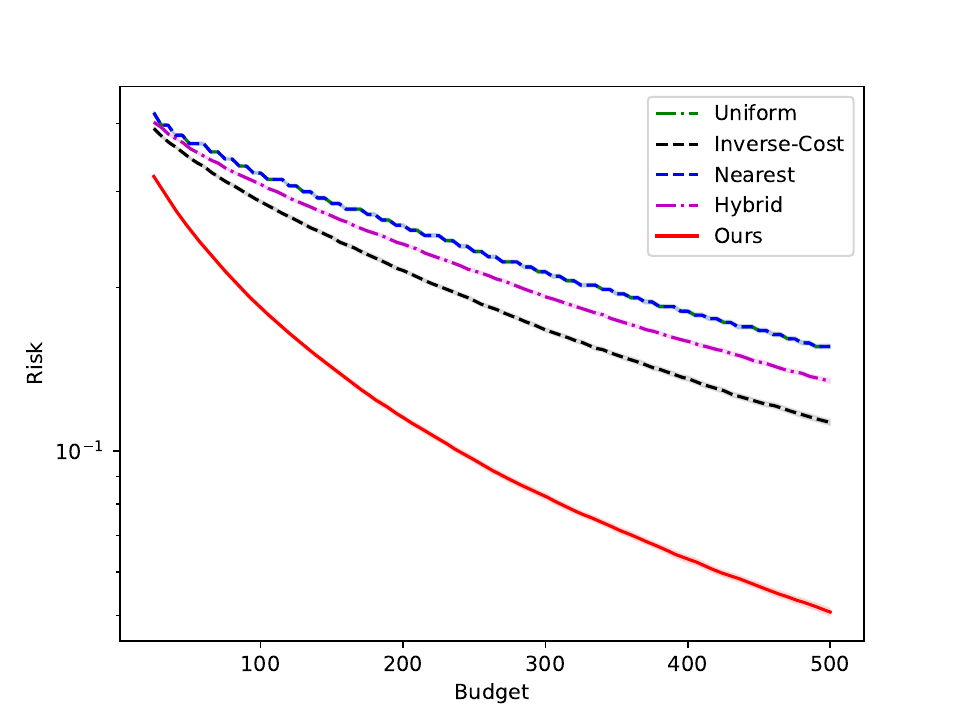}
	\end{minipage}
	\hspace*{\fill}
    \caption{Estimated risk based on 100 simulations in each setting. Error regions represent empirical average $\pm$ 2 SE. Row 1: Population mean. Row 2: Binary classification.}
    \label{fig:pyrm}
\end{figure}
}

\begin{abstract}
Data collection is a critical component of modern statistical and machine learning pipelines, particularly when data must be gathered from multiple heterogeneous sources to study a target population of interest. In many use cases, such as medical studies or political polling, different sources incur different sampling costs. Observations often have associated group identities---for example, health markers, demographics, or political affiliations---and the relative composition of these groups may differ substantially, both among the source populations and between sources and target population.

In this work, we study multi-source data collection under a fixed budget, focusing on the estimation of population means and group-conditional means. We show that naive data collection strategies (e.g. attempting to ``match'' the target distribution) or relying on standard estimators (e.g.\ sample mean) can be highly suboptimal. Instead, we develop a sampling plan which maximizes the \emph{effective sample size}---the total sample size divided by 
$\DD{\chi^2}{q}{\overline{p}} + 1$,
where $q$ is the target distribution, $\overline{p}$ is the aggregated source distributions, and $D_{\chi^2}$ is the $\chi^2$-divergence.
We pair this sampling plan with a classical post-stratified estimator and upper bound its risk. We provide matching lower bounds, establishing that our approach achieves the budgeted minimax optimal risk. Our techniques also extend to prediction problems when minimizing the excess risk, providing a principled approach to multi-source learning with costly and heterogeneous data sources.\footnote{%
    Accepted for presentation at the Conference on Learning Theory (COLT) 2026. Appears as an extended abstract in PMLR, volume 336.
}
\end{abstract}

\section{Introduction}
\label{sec:intro}

Data collection is a central  component of any data analysis pipeline. The performance of even the most well-designed estimators (or learning algorithms) depends heavily on the data they are trained on, and thus the design of the data collection scheme can be just as important as the estimator itself.

\paragraph{Multi-source data collection} In many practical scenarios, data is collected from multiple \emph{sources} in service of designing a system to better understand a \emph{target} population of interest.
For example, planners of a clinical study may work with a number of treatment centers (sources) spread across a country, with the goal of predicting treatment effects across the entire population (target) of their country.
Complicating this process is the heterogeneity in the distribution of \textit{groups}, such as demographics or disease prevalence, at the sources, and the data collection costs at each source. These groups often vary 
substantially across centers and  differ markedly from those of the overall population, and costs of data collection can vary due to operational expenses and participant recruitment.

Most results in the statistics and machine learning literature rely on the assumption that the training data come from the same target distribution against which an estimator will be evaluated.
Thus, when faced with the problem of designing a data collection system with access to multiple, heterogeneous data sources, at first glance, it would appear that the goal should be to craft a sampling scheme which results in an aggregated ``source'' which most closely resembles the target distribution of interest. However, this naive approach ignores differences in sampling costs across sources, and fails to exploit the fact that a well-designed estimator should benefit from additional  data.
On the other hand, simply maximizing the number of collected samples without regard to the target distribution is also not meaningful.
We illustrate these challenges via the example below.

\begin{example}
We wish to estimate the average BMI of adults in a state where it is known that 25\% of adults are physically active (A) and 75\% are inactive (I). Collecting i.i.d.\ samples from the population would permit direct estimation via the sample mean. However, data must be collected from one of two sources, an urban and rural clinic, under a fixed budget of \$1{,}000. Measuring BMI costs \$1 and \$2 per sample at the urban and rural clinics, respectively. The urban population is 80\% active, while the rural population matches the state (25\% active), reflecting strong selection effects. Crucially, an individual’s activity status is unknown at recruitment, requiring a post-measurement questionnaire.

Allocating the entire budget to the urban clinic yields 1{,}000 measurements, but only roughly 200 from group I, which represents the majority of the population of interest.
Another naive approach attempting to ``match'' the state distribution would allocate the entire budget to the rural clinic, yielding 500 measurements matching the state (roughly 125 from A and 375 from I), and use the sample mean.
However, using the techniques we develop, the optimal allocation turns out to be collecting 152 and 424 samples from the urban and rural clinics, respectively, yielding roughly 228 samples from A and 348 samples from I.
While there is a mismatch with the state population, collecting more data from group A (relative to ``matching'') is more helpful here. Though the sample mean is no longer appropriate, pairing this with an appropriate estimator could leverage the larger, albeit biased, sample to improve estimation accuracy.
\end{example}

This example highlights several key features of multi-source data collection:
\emph{(i)} Sampling from different sources incurs different costs.
\emph{(ii)} Source populations have heterogeneous group compositions that can differ substantially from the target population.
\emph{(iii)} Group proportions are \emph{known}, both at the source and target levels, but it is not always possible to cheaply preselect individuals based on group identity.\footnote{%
In practice, planners of clinical studies may have already pre-selected based on easily observable characteristics (e.g., age or race), while identifying other group attributes requires incurring the measurement cost (e.g., health conditions).
}
\emph{(iv)} Effective data collection must be paired with a well-designed estimator that can appropriately leverage large but systematically biased datasets. These challenges motivates the following question of study for this setting:
\begin{center}
	\itshape
	What is the optimal procedure for collecting and learning from data coming from sources with heterogeneous population compositions and unequal data collection costs? 
\end{center}

\subsection{Model}
\label{sec:model}

\paragraph{Environment}
A population can be divided into $K$ groups. A planner has access to $M$ sources, from which they can sample data. Observations are in the form of a tuple $(Z,Y)\,$, where $Z\in[K]=\{1,\ldots,K\}$ is the group identity and $Y\in\R$ is the response, or label. At source $m$, we assume $(Z,Y)\iidsim\PSm(z,y) = \qSm(z)\PYZ(y\mid z)\,$. Here, $\qSm$ is a \emph{known}\footnote{%
\label{foot:robust}%
We note that this assumption allows for clarity of our results, but is not critical. If instead $\qSm$ and $\qT$ are estimated to within $\epsilon$ relative-error, the performance of our proposed policy computed with the perturbed versions can be at most $(1+6\epsilon+\bigO(\epsilon^2))$-times worse than if computed with the true distributions. One could also utilize robust optimization techniques over an uncertainty region for $\qSm$ and $\qT\,$, but this is beyond the scope of this work.}
discrete distribution over $[K]\,$, defining the group distribution at source $m\,$, while $\PYZ$ is an \emph{unknown} conditional distribution.
All groups have positive probability in at least one source, \ie for each $z\in[K]\,$, there exists $m\in[M]$ such that $\qSm(z)>0$.
We assume that $\PYZ$ does not depend on the source $m$;
this assumption is reasonable in practice with sufficiently informative groups (e.g., treatment response is independent of location, conditional on disease condition, genetic markers, age, etc.), while still yielding a rich theoretical framework.
We also assume that $\PYZ$ belongs to the following class:
\begin{equation*}
	\meandist \defeq \left\{\PYZ \in \mathscr{P}(\R): \;
	\abs{\E[Y\mid Z]}\leq R\,,\;
    \Var(Y\mid Z)\leq\sigma^2,\;\text{a.s. }\right\}\,,
\end{equation*}
where $\mathscr{P}(\R)$ is the set of all probability distributions over $\R$.
Here, $R$ and $\sigma^2$ are \emph{unknown} to the policy.
As we show in Theorem~\ref{thm:imposs},  boundedness of $\abs{\E[Y\mid Z]}$ is necessary in this problem.

\paragraph{Learning with respect to a target group distribution}
In this work, we focus mainly on estimating \emph{(i)} the mean or \emph{(ii)} the vector of conditional means.
We formulate \emph{(i)} as estimating the expectation of $Y$, under the squared loss,
for a target distribution $\PT$ of interest.
The target distribution can be written as $\PT(z,y) = \qT(z)\PYZ(y\mid z)\,$; here $\qT$ is \emph{known},\textsuperscript{\ref{foot:robust}}
and defines the distribution of groups at the population level, and the conditional distribution $\PYZ$ is unknown and the same as the sources.
The population mean is $\thetaPM(\PT) = \E_{Y\sim \PT}[Y]$,
where $\thetaPM:\meandist\rightarrow \R$.
For \emph{(ii)}, we define the group-conditional means
as $\thetaGM(\PT) = \{\E_{Y,Z\sim\PT}[Y|Z=z]\}_{z\in[K]}$, where $\thetaGM:\meandist\rightarrow \R^K$, which we will estimate under the $\ell_2^2$ loss; as we will see, estimating the conditional means can be framed as estimation under a ``uniform'' target distribution.

These two parameters cover a broad range of quantities of interest to practitioners. The population mean corresponds to quantities such as the average treatment effect (ATE)~\citep{Hirano2003Efficient,Imbens2004Nonparametric,Chernozhukov2018Double} of a new therapeutic, proportion of votes for a candidate, or the expected revenue for a new product across an entire customer base. Likewise, the vector of conditional means corresponds to quantities such as the conditional average treatment effect (CATE)~\citep{Imai2013Heterogeneity,Wager2018Estimation,Nie2021Quasi}, the proportion of votes in each demographic group, or the expected revenue within each customer segment.

\paragraph{Multi-source data collection under a budget} 
In practice, there are important resource constraints to consider when collecting data. Traditionally, this is studied via the performance of the system in terms of the total sample size, but this only serves as a proxy for the actual constraints, such as time and/or money needed to sample data. To this end, we study the setting where each sample from source $m$ comes at a cost $c_m\,$, and so the total cost of a \emph{sampling plan} $\vn=(n_1,\ldots,n_M)\,$, where $n_m$ is the number of points to be collected from source $m\,$, is $\vc^\top\vn\,$, for $\vc = (c_1,\ldots,c_M)\,$.
We have a fixed budget $B>0\,$, and feasible sampling plans are those which satisfy $\vc^\top\vn\leq B\,$.

For population mean estimation, a policy is a tuple $(\vn,\thetahatpm)$, consisting of both a sampling plan $\vn$ and an estimator $\thetahatpm$ mapping the collected data $D$ to an estimate $\thetahatpm(D)$.
The risk of a policy under the (unknown) conditional distribution $\PYZ$ is the expected squared loss:
\begin{equation*}
		\riskpm((\vn,\thetahatpm),\PYZ) \defeq \E_{D\sim\PSjoint}\left[\left(\thetahatpm(D)-\thetaPM(\PT)\right)_2^2\right]\,.
\end{equation*}
Here, $\PSjoint$ denotes the joint distribution of data collected from the sources under the sampling plan $\vn$.
For simplicity, we suppress the dependence of the risk on the group distributions $\qSm$ and $\qT$, which are fixed and known to the planner ahead of time.
A policy’s performance is evaluated relative to the budgeted minimax risk $\mmriskpm$, defined below.
Note that the supremum is taken only over the conditional distributions $\PYZ$ which is the only unknown. We have:
\begin{equation}
	\begin{split}
		& \mmriskpm(B,\meandist) \defeq \inf_{\vc^\top\vn\leq B}\;\inf_{\thetahatpm}\;\sup_{\PYZ\in\meandist}\riskpm((\vn,\thetahatpm),\PYZ)\,,
	\end{split}
    \label{eqn:mmriskpm}
\end{equation}
For estimating group-conditional means, a policy is similarly a tuple $(\vn,\thetahatgm)\,$.
The risk of a planner's policy $\riskgm$ is as defined below, with the minimax risk defined similar to~\eqref{eqn:mmriskpm}. We have:
\begin{align*}
		\riskgm((\vn,\thetahatgm),\PYZ) \defeq \E_{D\sim\PSjoint}\left[\left\|\thetahatgm(D)-\thetaGM(\PT)\right\|_2^2\right]\,.
\end{align*}

\subsection{Summary of our contributions and techniques}

We first establish lower bounds on the minimax risks for this novel setting, and then design minimax optimal policies. We then extend these ideas to prediction tasks.
We now outline our main contributions, focusing on population mean estimation for simplicity.

\paragraph{Effective sample size}
A key quantity in our analysis is the \emph{effective sample size}, which quantifies how well we can estimate a population quantity when data comes from different sources.
To define this, for a sampling plan $\vn$, let $\qavg = (\vone^\top \vn)^{-1}\sum_{m=1}^Mn_m \qSm$ be the  mixture distribution over the group identities.
For a fixed set of source group distributions $\{\qSm\}_{m\in[M]}$,
we define the effective sample size $\neff(\vn, q)$ for a sampling plan $\vn$ with respect to a target distribution $q$ over groups as:
\begin{align}
	\neff(\vn, q) \defeq
	\frac{\vone^\top \vn}{\dd{q}{\qavg}}\,,
	\qquad
	\text{where,}
	\quad \dd{a}{b} \defeq \sum_{z\in[K]}\frac{a^2(z)}{b(z)}\,.
	\label{eqn:neff}
\end{align}
That is, the effective sample size is the total number of samples collected divided by the discrepancy measure $\dd{q}{\qavg}$.
The discrepancy measure satisfies $\dd{a}{b}\geq1\,$, with equality only when $a=b\,$. It also satisfies $\dd{a}{b}=\exp(\DD{2}{a}{b}) = \DD{\chi^2}{a}{b}+1\,$, where $D_2$ is the R\'enyi-2 divergence \citep{Rényi1961MeasuresEntropy} and $D_{\chi^2}$ is the $\chi^2$-divergence. 

\paragraph{Lower bound (\S~\ref{sec:me-lbs})}
Our first result establishes a lower bound on the risk achievable by any policy.
As we will see, this lower bound
also informs the design of an optimal policy for this problem.
To state the result, define
$\avgcost(\vn) = (\vone^\top \vn)^{-1}\vc^\top \vn$ as the \emph{average sample cost}
of a sampling plan $\vn$.
\theoremstyle{plain}
\newtheorem*{thm:me-lb}{Theorem~\ref{thm:me-lb}}
\begin{thm:me-lb}[Informal]
	The minimax risk~\eqref{eqn:mmriskpm} satisfies the following lower bound, where $\nstarT$ is the sampling plan $\vn$ which maximizes
	$\neff(\vn, \qT)$ subject to the constraint $\vc^\top \vn \leq B$.
	We have,
	\begin{align*}
		\mmriskPM(B,\meandist) & \geq \frac{\sigma^2\avgcost(\nstarT)\avgweightbeta[\nstarT]}{B} - \littleO\left(\frac{1}{B}\right)
		= \frac{\sigma^2}{\neff(\nstarT, \qT)} - \littleO\left(\frac{1}{B}\right),
	\end{align*}
\end{thm:me-lb}
The lower bound illustrates that the effective sample size is a key quantity in this problem, as we see the familiar $\sigma^2/n$ bound on the risk, except with $n$ replaced by the effective sample size of $\nstarT$.

\emph{Proof outline.}
The proof, which is a key technical contribution of this work, builds on the common technique of lower bounding the worst-case risk over $\meandist$ by the expected risk of the Bayes estimator under a suitably chosen prior.
However, the typical choice for this technique, leveraging normal-normal conjugacy, does not work in our setting, as we assume a bounded domain for the means (recall that the problem is hopeless without boundedness). We instead utilize a uniform prior, requiring study of the expected variance of a truncated normal posterior distribution.
We then recover appropriate dependence on $\neff(\vn,\qT)$ via a decomposition of an intractable integral
appearing in the posterior variance into 3 regions, carefully chosen to leverage normal distribution properties and a Gaussian tail lower bound technique.

\paragraph{Upper bound (\S\ref{sec:mean-est})}
This lower bound, if tight, suggests collecting data according to the sampling plan $\nstarT$ and pairing it with an appropriate estimator to obtain a minimax-optimal policy.
The sample mean is inadequate here, since the mixture distribution induced by $\nstarT$ generally differs from the target distribution.
Instead, we show that the natural and classically studied post-stratified estimator $\thetahatPS$~\citep{Holt1979PostStratification}, which first estimates the mean within each group $\{\ybarz\}_{z\in[K]}$ and then combines them as $\sum_{z\in[K]} \qT(z)\,\ybarz$, is optimal for this setting. 
This yields the following theorem.
\theoremstyle{plain}
\newtheorem*{thm:me-strat-ub}{Theorem~\ref{thm:me-strat-ub}}
\begin{thm:me-strat-ub}[Informal]
	The policy $(\nstarT,\thetahatPS)$ 
	achieves risk,
	\begin{align*}
		\riskPM((\nstarT,\thetahatPS),\PYZ) & \leq \frac{\sigma^2\avgcost(\nstarT)\avgweightbeta[\nstarT]}{B} + \littleO\left(\frac{1}{B}\right)
		= \mmriskPM(B,\meandist) + \littleO\left(\frac{1}{B}\right)
	\end{align*}
\end{thm:me-strat-ub}
This result matches the lower bound with exact constants in the leading term, establishing that $(\nstarT,\thetahatPS)$ is minimax optimal. While relying on standard tools---a careful decomposition of the risk and Taylor expansion to control lower-order terms---this proof is nonetheless novel, differing fundamentally from prior analyses of the post-stratified estimator.

\paragraph{Prediction problems (\S~\ref{sec:prediction})} 
Next, we explore prediction problems where we have additional features $X$, associated with each observation, separate from the group identity $Z$.
Given a hypothesis class $\cH$ consisting of hypotheses mapping $(Z, X)$ to a label $Y$,
we wish to collect data, and use it to find a hypothesis $h\in\cH$ which minimizes the prediction error with respect to a given target distribution. 

\emph{{Upper bound.}}
We first study how our sampling plan, which maximizes the effective sampling size,
performs when paired with an
Importance-Weighted Empirical Risk Minimization (IWERM) procedure $\hiwerm$~\citep{Cortes2010LearningBounds}.
We have the following bound on the excess risk of our method $\riskPR((\nstarT,\hiwerm),\cH)$ (\ie risk of our method minus the best achievable in $\cH$).

\theoremstyle{plain}
\newtheorem*{thm:avg-risk-ub}{Theorem~\ref{thm:avg-risk-ub}}
\begin{thm:avg-risk-ub}[Informal]
	Suppose the hypothesis class $\cH\subset\{h:[K]\times\cX\to\R\}\,$, when restricted to any $z\in[K]\,,$ has finite pseudo-dimension $\pseudodimH$ over $\cX\,$.
	Then, under a budget $B$, the excess risk of our policy $(\nstarT,\hiwerm)$ can be upper bounded by
	\begin{align*}
		\riskPR((\nstarT,\hiwerm),\cH) \in \bigOtilde\left(\sqrt{\frac{\pseudodimH K\avgcost(\nstarT)\avgweightbeta[\nstarT]}{B}}\right)
	\end{align*}
\end{thm:avg-risk-ub}
		
\emph{{Lower bound.}}
To study if our sampling plan is optimal in a prediction setting, we establish lower bounds
for binary classification under the 0--1 loss, where the pseudo-dimension is simply the VC dimension~\citep{Vapnik1971UniformConvergence,Pollard1984ConvergenceStochastic}.
This gives us the following result
on the minimax excess risk $\mmriskPR(B,\cH)$ relative to a hypothesis class $\cH$ under a budget $B$.

\theoremstyle{plain}
\newtheorem*{thm:avg-risk-lb}{Theorem~\ref{thm:avg-risk-lb}}
\begin{thm:avg-risk-lb}[Informal]
	Suppose the hypothesis class $\cH\subset\{h:[K]\times\cX\to\{\pm1\}\}\,$, when restricted to any $z\in[K]\,,$ has finite VC-dimension $\vcdimH$ over $\cX\,$. 
	Let $\qmin = \min_{z\in[K]}\qT(z)\,$.
	Then, the minimax excess risk satisfies
	\begin{align*}
		\mmriskPR(B,\cH) \in \Omega\left(\sqrt{\frac{\vcdimH\qmin\avgcost(\nstarT)\avgweightbeta[\nstarT]}{B}}\right)
	\end{align*}
\end{thm:avg-risk-lb}
Comparing with the upper bound, we have matching dependence on $\vcdimH\neff^{-1}(\nstarT,\qT)$, indicating that the effective sample size is a fundamental quantity in this setting as well.
While there is a $\sqrt{K/\qmin}$ gap between the upper and lower bounds, we believe this is an artifact of our analysis.

As was the case for mean estimation, the key technical challenge in this setting is the lower bound.
Despite the wealth of related covariate shift literature \citep{Shimodaira2000Improvingpredictive,Sugiyama2007CovariateShift,Mansour2009Multiplesource,Cortes2010LearningBounds,Cortes2019Relativedeviation,Hanneke2019ValueTarget,Sugiyama2019MachineLearning,Zhang2020OnestepApproach,Kpotufe2021Marginalsingularity,Fang2023GeneralizingImportance,Ma2023Optimallytackling,Ge2024MaximumLikelihood}, to the best of our knowledge, we are the first to provide a minimax lower bound for the excess risk with explicit dependence on the discrepancy $\avgweight\,$.

\emph{Proof outline.}
The standard reduction from learning to testing to prove lower bounds does not apply directly in our setting,
since training data are drawn from source distribution(s) that differ from the target
distribution under which performance is evaluated.
To address this challenge, we develop a novel \emph{reduction-to-testing} lemma that
explicitly accounts for this distributional mismatch.

We combine this lemma with Fano's inequality to obtain a general framework for minimax
lower bounds. The framework relates the minimax risk to
\emph{(i)} the separation of losses
induced under the target distribution and
\emph{(ii)} the KL divergences induced between the
corresponding source distributions.
To fully exploit this machinery and induce the desired dependence on $\avgweight$,
we construct a family of conditional distributions indexed by a subset of the
$d$-dimensional hypercube via the Gilbert--Varshamov lemma~\citep{Gilbert1952,Varshamov1957}.
The construction separates positive and negative class probabilities according to the
group identity $z$.
With an appropriate choice of separations, the resulting class is well separated in loss
under $\qT$ while maintaining uniformly bounded KL divergences under $\qavg$, yielding the
desired lower bound.

\paragraph{Empirical evaluation (App.~\ref{sec:experiments})}
We corroborate our results in simulations, comparing to a set of straightforward, yet suboptimal, alternative sampling plans. We demonstrate how they under-perform compared to our policy of maximizing the effective sample size.
\subsection{Related Work}
\label{sec:rel-work}
\paragraph{Sampling techniques} Existing methods broadly fall into the following three categories.

\emph{Stratified sampling.}
The most traditional approach is stratified sampling, dating back to \cite{Neyman1934TwoDifferent}. Here, it is assumed that the planner can collect \iid samples from each group (stratum). In our setting, this approach would require recruited individuals to be pre-selected based on their group identity, or each source to itself be a group. However, it is often impossible to observe the group identity without incurring the sampling cost, and we aim to study groups which are not uniquely defined by the sources. The stratified setting is well studied and continues to attract new research \citep{Khan2003TheoryMethods,Meng2013ScalableSimple,Khan2015Designingstratified,Liberty2016StratifiedSampling,Cervellera2018DistributionPreservingStratified,Sauer2021Optimalallocation,Yadav2025OptimalStrategy}, but these approaches fail to generalize to our setting.

\emph{Multiple-frame sampling.}
Multiple-frame sampling studies the setting where the sources (frames) from which the planner is sampling cover overlapping sub-populations. A classic example is a telephone survey \citep{Wolter2015Optimumallocation}, where the two sources are cell phone and landline users. Some individuals will own both, so they have a possibility of being selected in both sources. Optimal allocation strategies in this setting are thus focused on the size and variance within these intersections, which is not applicable to our setting. We refer the reader to a thorough review from \cite{Lohr2007RecentDevelopments}.

\emph{Cluster sampling.}
Of these approaches, cluster sampling is the most similar to our problem. Here, the planner collects \iid samples from each source (cluster), as in our setting. However, these approaches often assume that sources are sampled from the super-population of sources, \eg randomly selecting blocks in a city \citep{Tryfos1996Sampling}. Instead, our setting allows for the common case where sources available for sampling are fixed a priori. Even when this assumption is not critical, optimal allocation in the cluster sampling literature does not address our setting of heterogeneous group compositions and how they affect the design of sampling plans \citep{Connelly2003Balancingnumber,Sharma2015Determiningoptimum,Shen2020OptimalSample,Copas2021Optimaldesign,Varshney2023Optimumallocation}.

\paragraph{Effective sample size}
We note that our definition of the effective sample size appears in related fields, primarily in evaluating MCMC algorithms in Bayesian importance sampling  \citep{Kong1994SequentialImputations,Fishman1996MonteCarlo,Liu1996Metropolizedindependent,Agapiou2017ImportanceSampling,Martino2017Effectivesample,Elvira2022RethinkingEffective} and in learning under covariate shift \citep{Mansour2009Multiplesource,Cortes2010LearningBounds,Cortes2019Relativedeviation,MaiaPolo2023Effectivesample}. Further, existing lower bounds in the covariate shift literature either do not consider a fixed target-source pair \citep{Hanneke2019ValueTarget,Kpotufe2021Marginalsingularity,Ma2023Optimallytackling} or do not induce dependence on $\avgweight$ when applied to our setting \citep{Ge2024MaximumLikelihood}.

\paragraph{Data procurement}
There is a line of work in the EconCS literature sudying data procurement, which is conceptually related to our problem of study, yet addresses distinct questions and problem settings from our work.~\citet{Abernethy2015LowCostLearning} and~\citet{Chen2018OptimalData} address the problem of sequentially purchasing data from strategic agents subject to an expected budget constraint for prediction and parameter estimation, respectively. In contrast, our setting considers data sources with fixed, known procurement costs and an exact budget constraint. Further, our setting allows for multiple, differing source distributions and a separate target distribution, whereas~\cite{Abernethy2015LowCostLearning} and~\cite{Chen2018OptimalData} consider a single, shared data distribution for all agents, which is also the target distribution.~\citet{Haghtalab2022OnDemandSampling} and~\citet{Aznag2023activelearning} consider more similar data-generating processes to our setting for prediction and parameter estimation, respectively, having distributions that differ over stratified population groups. However, neither address the aspects of multiple data sources or source-dependent costs, instead considering the ability to sample directly from a chosen group and acquisition of a fixed $T$ number of data points.
\section{Lower Bounds}
\label{sec:me-lbs}
We now study the mean estimation problems from~\S\ref{sec:model}. Recall the definition of the minimax risk from~\eqref{eqn:mmriskpm}, and the effective sample size $\neff$ from~\eqref{eqn:neff}. Let $\unif$ denote the uniform distribution over $[K]$ and $\nz$ denote the (randomly) observed number of observations from group $z$ in a dataset $D\,$. We begin with the following lower bounds and a sketch of our proof; the full version is in App.~\ref{sec:pfs-me-lb}.

\begin{restatable}{theorem}{melb}
	\label{thm:me-lb}
	Fix a budget $B>0$ and a vector of costs $\vc=(c_1,\ldots,c_M)\,$, with $c_m>0$ for all $m\in[M]\,$. Define the following sampling plans,
	\begin{align}
		\label{eq:max-neff}
		\nstarT \in \argmax_{\vn\in\N^M}\neff(\vn,\qT)\;\;\text{\st}\;\;\vc^\top\vn\leq B\,, \qquad 
        \nstarU \in \argmax_{\vn\in\N^M}\neff(\vn,\unif)\;\;\text{\st}\;\;\vc^\top\vn\leq B
	\end{align}
	Then, we have the following lower bounds on the risk of any policy $(\vn,\theta)\,$,
	\begin{align*}
		\mmriskPM(B,\meandist) & \geq \frac{\sigma^2\avgcost(\nstarT)\avgweightbeta[\nstarT]}{B} - \bigO\left(\frac{1}{B^{\nicefrac{3}{2}}}\right) = \frac{\sigma^2}{\neff(\nstarT,\qT)} - \littleO\left(\frac{1}{B}\right)\\
		\mmriskGM(B,\meandist) & \geq \frac{K^2\sigma^2\avgcost(\nstarU)\avgweightUbeta[\nstarU]}{B} - \bigO\left(\frac{1}{B^{\nicefrac{3}{2}}}\right) = \frac{K^2\sigma^2}{\neff(\nstarU,\unif)} - \littleO\left(\frac{1}{B}\right)
	\end{align*}
\end{restatable}
\begin{proof}(Proof sketch of Theorem~\ref{thm:me-lb}).
	We consider a mean parameter $\vmu\in[-R,R]^K$ and a uniform prior distribution over the value of this parameter, $\Pi=\Unif([-R,R]^K)\,$. We consider the distributions $P_{\vmu}\in\meandist$ having normal conditional distributions with conditional mean vector $\vmu$ and variance $\sigma^2\,$; that is, $Y\sim P_{\vmu}\implies Y\mid Z=z\sim N(\muz,\sigma^2)\,$. 
	We first lower bound the worst-case risk by the expected risk of the Bayes estimator.
	\begin{restatable}{lemma}{mmtobayes}
		\label{lem:mm-to-bayes}
		Consider the following generative model: first, $\vmu\sim\Pi\,$, then, for a sampling plan $\vn\,$, the dataset $D\sim\PSjoint\,$, with $\PYZ=\Pmu\,$. Further, denote by $\PSjointavg$ the \textit{unconditional} data distribution, accounting for the randomness in $\mu\,$, and by $\postpi$ the posterior distribution over $\vmu\,$. Then,
        \begingroup
        \allowdisplaybreaks
		\begin{align}
			\label{eq:cdtn-var}
			\begin{split}
				\mmriskPM(B,\meandist) & \geq \inf_{\vc^\top\vn\leq B}\sum_{z\in[K]}\qT^2(z)\E_{D\sim\PSjointavg}\left[\Var_{\vmu\sim\postpi}(\muz\mid D)\right] \\
				\mmriskGM(B,\meandist) & \geq \inf_{\vc^\top\vn\leq B}\sum_{z\in[K]}\E_{D\sim\PSjointavg}\left[\Var_{\vmu\sim\postpi}(\muz\mid D)\right]
			\end{split}
		\end{align}
        \endgroup
	\end{restatable}
	We then study the posterior distribution $\postpi\,$. We find that the $\muz$'s are independent and, on the event $\nz>0\,$, $\muz$ follows a truncated normal distribution with location $\ybarz\,$, scale $\nicefrac{\sigma^2}{\nz}\,$, and domain $[-R,R]\,$. Next, we explicitly compute the density of $\PSjointavg(\ybarz\mid\nz)$ and combine it with the variance of a truncated normal distribution to arrive at,
	\begin{align*}
		\E_{D\sim\PSjointavg}&\left[\Var_{\vmu\sim\postpi}(\muz\mid D)\mid\nz\right] \\
		&\qquad\qquad\qquad
		= \frac{\sigma^2}{\nz}\left(1- \frac{\sigma}{2\sqrt{\nz}R}\int_{\R}\frac{\left[\phi\left(x+\frac{\sqrt{\nz}R}{\sigma}\right)-\phi\left(x-\frac{\sqrt{\nz}R}{\sigma}\right)\right]^2}{\Phi\left(x+\frac{\sqrt{\nz}R}{\sigma}\right) - \Phi\left(x-\frac{\sqrt{\nz}R}{\sigma}\right)}\,dx\right)\,,
	\end{align*}
	where $\phi$ and $\Phi$ are the standard normal pdf and cdf, respectively. Clearly, because the variance is non-negative, this integral is bounded above by $\frac{2\sqrt{\nz}R}{\sigma}\,$, but if we wish to recover the appropriate leading term in the lower bound, we need to show it is $\in\littleO(\sqrt{\nz})\,$. In fact, we show that it is bounded by a constant independent of $\nz\,$.
	\begin{restatable}{lemma}{integral}
		\label{lem:integral}
		For any $C>0\,$, we have
		\begin{align*}
			\int_{\R}\frac{(\phi(x+C)-\phi(x-C))^2}{\Phi(x+C)-\Phi(x-C)}\,dx \leq 8
		\end{align*}
	\end{restatable}
	We prove this intermediate result by recognizing that this function is even, then dividing the positive real line into two or three regions, depending on the value of $C\,$. When $C>3\,$, for $x<C-3$, we use the commonly known fact that $\Phi(3)-\Phi(-3)>0.997$ to lower bound the denominator, then analytically compute the integral of the numerator, finding it no larger than $(1.994\sqrt{\pi})^{-1}\,$. For $x\in[\max\{C-3,0\},C+3]\,$, we utilize the fact that the integrand is never larger than $\nicefrac{1}{2}$ to bound this region by 3. Finally, for $x>C+3\,$, we use a technique from \cite{Vershynin2018HDP} to show,
	\begin{align*}
		\Phi(x+C) - \Phi(x-C) \geq \frac{\phi(x-C)-\phi(x+C)}{2(x-C+1)}\,,
	\end{align*}
	for $x>C+3\,$. Using this to lower bound the denominator, we can analytically compute the resulting integral, and find it is no larger than 0.02 in this region, giving the overall bound of 8. Plugging this result back into \eqref{eq:cdtn-var} and using Jensen's inequality to lower bound $\E[\nz^{-1}]\geq(\E[\nz])^{-1}\,$, we get,
	\begin{align*}
		\mmriskPM(B,\meandist)&\geq\inf_{\vc^\top\vn\leq B}\frac{\sigma^2\avgweight}{\vone^\top\vn} - \frac{4\sigma^2}{R(\vone^\top\vn)^{\nicefrac{3}{2}}}\sum_{z\in[K]}\frac{1}{\qavg^{\nicefrac{3}{2}}(z)} \\
		\mmriskGM(B,\meandist)&\geq\inf_{\vc^\top\vn\leq B}\frac{K^2\sigma^2\avgweightU}{\vone^\top\vn} - \frac{4\sigma^2}{R(\vone^\top\vn)^{\nicefrac{3}{2}}}\sum_{z\in[K]}\frac{\qT^2(z)}{\qavg^{\nicefrac{3}{2}}(z)}
	\end{align*}
	Finally, we recognize that choosing $\vn$ to minimize the leading terms can only possibly incur $\bigO(B^{-\nicefrac{3}{2}})$ additional risk beyond the optimizer of the entire expression, and $\nstarT$ and $\nstarU$ do this by definition. Further, $\nstarT$ and $\nstarU$ exhaust the entire budget, and so $\vone^\top\nstarT=\frac{B}{\avgcost(\nstarT)}\,$, and likewise for $\nstarU\,$.
\end{proof}

\section{Method}
\label{sec:mean-est}

We begin with the design of an estimator for a given sampling plan in \S~\ref{ss:me-est}, then study the optimal data collection scheme in \S\ref{ss:me-ess}.

\subsection{Estimator design}
\label{ss:me-est}
We find it instructive to begin our study with the design of an estimator under an arbitrary sampling plan $\vn\,$. Taking cues from the Bayes estimators used in the proof of Theorem~\ref{thm:me-lb}, for estimating the population mean, we propose the use of the classical post-stratified mean estimator of \cite{Holt1979PostStratification}. This estimator stratifies the collected data by group identity, estimates the group conditional means separately, and re-weights them based on $\qT$:
\begin{equation*}
	\thetahatPS \defeq \sum_{z\in[K]}\qT(z)\ybarz,
\quad \text{where,}\quad \ybarz = \frac{1}{n_z}\sum_{i=1}^{\vone^\top\vn}\1_{\{z\}}(Z_i)Y_i
\end{equation*}
Likewise, when estimating the vector of group-conditional means, we propose using the vector of observed group-conditional means, $\thetahatVM\defeq\{\ybarz\}_{z\in[K]}\,$. (By convention, $\ybarz=0$ when $\nz=0\,$.)

While the post stratified estimator is classical and intuitive, it has not been analyzed in a similar setting, and both foundational \citep{Holt1979PostStratification,Bethlehem1987Linear,Smith1991PostStratification} and contemporary \citep{Miratrix2013AdjustingTreatment} analyses study its performance conditioned on the $\nz$'s, which is not appropriate for our setting. 
We now present our results for the risk of our proposed estimators, $\thetahatPS$ and $\thetahatVM\,$. The proof of this theorem, which is straightforward, albeit new, is included in App.~\ref{sec:pfs-me-ubs}.

\begin{restatable}{theorem}{meestub}
	\label{thm:me-est-ub}
	Fix a sampling plan $\vn\in\N^M\,$, satisfying $\qavgbeta[\vn](z)>0$ for all $z\in[K]\,$. Then, there exist estimators $\thetahatPS$ and $\thetahatVM$ such that, for any $\PYZ\in\meandist\,$, the following holds,
    \begingroup
    \allowdisplaybreaks
	\begin{align*}
		\riskPM((\vn,\thetahatPS),\PYZ) & \leq \frac{\sigma^2\avgweight}{\vone^\top\vn} + \littleO\left(\frac{1}{\vone^\top\vn}\right) = \frac{\sigma^2}{\neff(\vn,\qT)} + \littleO\left(\frac{1}{\vone^\top\vn}\right) \\
		\riskGM((\vn,\thetahatVM),\PYZ) & \leq \frac{K^2\sigma^2\avgweightU}{\vone^\top\vn} + \littleO\left(\frac{1}{\vone^\top\vn}\right) = \frac{K^2\sigma^2}{\neff(\vn,\unif)} + \littleO\left(\frac{1}{\vone^\top\vn}\right)
	\end{align*}
    \endgroup
\end{restatable}

\subsection{Sampling plan design}
\label{ss:me-ess}
Based on Theorem~\ref{thm:me-est-ub}, we can minimize the upper bounds on the risk of our policies by choosing a sampling plan which maximizes the effective sample size within the allotted budget $B\,$. This also aligns with our results in Theorem~\ref{thm:me-lb}, being exactly $\nstarT$ and $\nstarU$ for estimating the population mean and the vector of group means, respectively. Thankfully, this choice is also practical, as $(\neff(\vn,\qT))^{-1}$ is convex in $\vn\,$. Given $\qT$ and $B\,$, finding $\nstarT$ is then equivalent to minimizing a convex function subject to the linear constraint, so it is simple and efficient to implement. The following is a direct consequence of Theorem~\ref{thm:me-est-ub} and the definitions of $\nstarT$ and $\nstarU\,$.
\begin{restatable}{theorem}{mestratub}
	\label{thm:me-strat-ub}
	Fix a budget $B>0$ and a vector of costs $\vc = (c_1,\ldots,c_M)\,$, with $c_m>0$ for all $m\in[M]\,$. There exist policies $(\nstarT,\thetahatPS),(\nstarU,\thetahatVM)$ that achieve,
	\begin{align*}
		\riskPM((\nstarT,\thetahatPS),\PYZ) & \leq \frac{\sigma^2\avgcost(\nstarT)\avgweightbeta[\nstarT]}{B} + \littleO\left(\frac{1}{B}\right) = \mmriskPM(B,\meandist) + \littleO\left(\frac{1}{B}\right) \\
		\riskGM((\nstarU,\thetahatVM),\PYZ) & \leq \frac{K^2\sigma^2\avgcost(\nstarU)\avgweightUbeta[\nstarU]}{B} + \littleO\left(\frac{1}{B}\right) = \mmriskGM(B,\meandist) + \littleO\left(\frac{1}{B}\right)
	\end{align*}
\end{restatable}
Importantly, we see that the leading terms of Theorem~\ref{thm:me-strat-ub} \emph{exactly} match the leading terms of Theorem~\ref{thm:me-lb}. This proves the minimax optimality of our approach, up to lower order terms.
\section{Prediction problems}
\label{sec:prediction}

As an extension of our mean estimation problem, we now consider the case where a planner wishes to learn a model to predict the response $Y\,$, based on the group identity $Z$ and features $X$.

\paragraph{Preliminaries}
In the prediction setting, observations are a triple $(Z,X,Y)\,$, where $Z$ and $Y$ remain the group identity and response, and the features $X$ belong to some metric space $\cX$.
The conditional distribution $\PYZ$ is replaced by $\PXY\,$, remaining fixed for all sources and target. We allow $\PXY$ to be any distribution over $\cX\times\R\,$, instead enforcing boundedness through the loss function, $\ell\,$. The map $\ell:\R^2\to[0,1]\,$, defines the loss incurred for predicting $y^\prime$ when the true response is $y\,$. A planner's policy is now a sampling plan $\vn$ and a \textit{model}, $\hhat\,$, mapping an observed dataset to a \textit{hypothesis}, $h:[K]\times\cX\to\R\,$, from some pre-specified hypothesis space $\cH\,$. The complexity of this class is controlled by the VC-dimension \citep{Vapnik1971UniformConvergence} for binary classification under the 0--1 loss, or the pseudo-dimension \citep{Pollard1984ConvergenceStochastic} for more general tasks. A planner's performance is measured by the excess risk of their policy, defined as,
\begin{align*}
	\riskPR((\vn,\hhat),\cH,\PXY)\defeq\E_{D\sim\PSjoint}\left[\E_{\PT}\left[\ell(\hhat_D(Z,X),Y)\right] - \inf_{h\in\cH}\E_{\PT}\left[\ell(h(Z,X),Y)\right]\right]\,,
\end{align*}
denoting by $\hhat_D$ the hypothesis $h\in\cH$ selected based on the dataset $D\,$. This is likewise measured against the minimax excess risk,
\begin{align*}
	\mmriskPR(B,\cH) \defeq \inf_{\vc^\top\vn\leq B}\inf_{\hhat}\sup_{\PXY\in\preddist}\riskPR((\vn,\hhat),\cH,\PXY)
\end{align*}

\subsection{Upper Bound}
We begin with a study of how our choice of sampling plan performs in the prediction setting. We pair our sampling plan $\nstarT$ with the following importance-weighted empirical risk minimization (IWERM) procedure,  which outputs a hypothesis as follows:
\begin{align*}
    \hiwerm \in \argmin_{h\in\cH} \frac{1}{\vone^\top\vn}\sum_{i=1}^{\vone^\top\vn}\frac{\qT(z_i)}{\qavg(z_i)}\ell(h(z_i,x_i),y_i)
\end{align*}
Similar procedures appear in previous studies under different settings~\cite{Cortes2010LearningBounds,Cortes2019Relativedeviation}.
They also draw a connection to the effective sample size, but they study a single-source-target setting and do not consider how a planner may impact the source distribution. %We bound the excess risk of this policy as follows.
\begin{restatable}{theorem}{predub}
	\label{thm:avg-risk-ub}
	Fix a budget $B>0$ and a vector of costs $\vc = (c_1,\ldots,c_M)\,,$ with $c_m>0$ for all $m\in[M]\,$. Let $\maxweight=\max_{z\in[K]}\nicefrac{\qT(z)}{\qavg(z)}\,$.
    Suppose the hypothesis class $\cH\subset\{h:[K]\times\cX\to\R\}\,$, when restricted to any $z\in[K]\,,$ has finite pseudo-dimension $\pseudodimH$ over $\cX\,$. Further, suppose we have $\ell(y,y^\prime)$ monotone in $\abs{y-y^\prime}$ and $\vone^\top\vn \geq \pseudodimH K\,$. Then, there exists a policy $(\nstarT,\hiwerm)$ such that, for all $\PXY\in\preddist\,$, we have,
	\begin{equation}
		\label{eq:ub-exp-risk}
		\begin{aligned}
			\risk((\nstarT,\hiwerm),\cH,\PXY) & \leq \log\left(\frac{e\maxweight[\nstarT]}{\sqrt{\avgweightbeta[\nstarT]}}\right)\sqrt{\frac{192\pseudodimH K\log\left(\frac{2e\vone^\top\nstarT}{\pseudodimH K}\right)}{\neff(\nstarT,\qT)}} \\
			& \qquad + \frac{64\pseudodimH K\avgcost(\nstarT)\maxweight[\nstarT]\log\left(\frac{2e\vone^\top\nstarT}{\pseudodimH K}\right)}{B}
		\end{aligned}
	\end{equation}
\end{restatable}
To prove this result, we use the fact that $\avgweight$ is the second moment of $\frac{\qT(z)}{\qavg(z)}$ under $\PSjoint$ to leverage results bounding the expected supremum of empirical processes from \cite{Baraud2016Boundingexpectation}. This result improves on the existing bounds for IWERM of \cite{Cortes2019Relativedeviation} by a $\sqrt{\log(\neff(\nstarT,\qT))}$ factor, at the expense of an added lower order term. The proof appears in App.~\ref{sec:pfs-learning-prediction}.

\subsection{Lower Bound}
To study the optimality of our proposed sampling plan, we develop lower bounds for binary classification under the 0--1 loss, where the pseudo-dimension and VC-dimension are equivalent. 
\begin{restatable}{theorem}{predlb}
	\label{thm:avg-risk-lb}
	Fix a budget $B>0$ and a vector of costs $\vc = (c_1,\ldots,c_M)\,,$ with $c_m>0$ for all $m\in[M]\,$. Let
	$\nstarT$ be defined the same as in Theorem~\ref{thm:me-lb}, and let $\qmin = \min_{z\in[K]}\qT(z)\,$. Suppose the hypothesis class $\cH\subset\{h:[K]\times\cX\to\{\pm1\}\}\,$, when restricted to any $z\in[K]\,,$ has finite VC-dimension $\vcdimH$ over $\cX\,$. Further suppose that we have $\vcdimH\geq 16\,,$ and $B$ is sufficiently large \st $B > \vcdimH\avgcost(\nstarT)(\nicefrac{\qT(z)}{\qavgbeta[\nstarT](z)})$ for all $z\in[K]\,$. Then, under the 0-1 loss, there exists a universal constant $C\,$, not depending on any problem parameters, such that,
	\begin{align*}
		\mmriskPR(B,\cH) \geq C\sqrt{\frac{\vcdimH\qmin\avgcost(\nstarT)\avgweightbeta[\nstarT]}{B}} = C\sqrt{\frac{\vcdimH\qmin}{\neff(\nstarT,\qT)}}
	\end{align*}
\end{restatable}
\begin{proof}(Proof sketch of Theorem~\ref{thm:avg-risk-lb}).
	We first construct a new framework for proving minimax lower bounds in settings where the source and target distributions differ, outlined in full in App.~\ref{sec:framework}. We then apply this to our specific setting, carefully constructing a sufficiently ``hard'' subclass of distributions to induce the appropriate dependence on $\qavg$ and $\qT\,$. The full proof is in App.~\ref{sec:pfs-pred-lb}.
    
    We begin by providing some necessary definitions. We denote the excess target population loss of a hypothesis $\hhat$ as $L(\hhat,\cH,\PT) = \E_{\PT}[\ell(\hhat(Z,X),Y)]-\inf_{h\in\cH}\E_{\PT}[\ell(h(Z,X),Y)]\,$. Then, for two different target distributions $\PTomega[1],\PTomega[2]\,$, we define their \textit{separation} \wrt $\cH$ as, 
	\begin{align*}
		\dsep{\PTomega[1]}{\PTomega[2]}\defeq\sup\{\delta\geq0:
		L(h,\cH,\PTomega[1])\leq\delta &\implies L(h,\cH,\PTomega[2])\geq\delta\;\forall\;h\in\cH\,, \\
		L(h,\cH,\PTomega[2])\leq\delta &\implies L(h,\cH,\PTomega[1])\geq\delta\;\forall\;h\in\cH\}\,.
	\end{align*}
	We call a collection of source-target pairs $\{(\PSomega[1],\PTomega[1]),\ldots,(\PSomega[N],\PTomega[N])\}$ \textit{target}-$\delta$-\textit{separated} when $\dsep{\PTomega[j]}{\PTomega[k]}\geq\delta$ whenever $j\neq k\,$. The key to our framework is the following lemma.
	\begin{restatable}[Multi-source-target reduction-to-testing]{lemma}{redtotest}
		\label{lem:red-test}
		Fix $\delta_{\vn}>0\,$, possibly depending on the sampling plan $\vn\,$, and a hypothesis class $\cH$. Let $\psi$ be a test, mapping a dataset $D$ to an index $j\in[N]\,$. If $\{(\PSomega[1],\PTomega[1]),\ldots,(\PSomega[N],\PTomega[N])\}$ is target-$\delta_{\vn}$-separated \wrt $\cH\,$, then,
		\begin{align*}
			\mmriskPR(B,\cH) \geq \inf_{\vc^\top\vn\leq B} \delta_{\vn}\inf_\psi\max_{j\in[N]}\PSomega[j](\psi(D)\neq j)
		\end{align*}
	\end{restatable}
    
	We then combine this result with Fano's inequality \citep{Fano1961Information} to recover the so-called ``Fano's method'' for minimax lower bounds in our setting: 
    \begin{align*}
        \mmriskPR(B,\cH)\geq\inf_{\vc^\top\vn\leq B}\delta_{\vn}\left(1-\frac{N^{-2}\sum_{j,k}\KL(\PSomega[j]\miid\PSomega[k])+\log(2)}{\log(N)}\right)
    \end{align*}
	
	This new framework provides the following intuition: we wish to construct a class of conditional distributions such that the source distributions are sufficiently ``close'' to one another, by their $\KL$-divergence, while the target distributions are as ``far'' as possible in their separation. It is this interplay between the roles of the target and source distributions, along with clever choices of conditional distributions, that allow us to induce the dependence on $\avgweight\,$. 
	
	We now construct our distribution class. Let $V=\vcdimH\,$. We use the fact that $\cH$ has VC-dimension $V$ over $\cX\,$, when restricted to any $z\in[K]\,$, to select sets $\shattz\subset\cX$ of size $V$ that are each \textit{shattered} by $\cH$ (all $2^V$ labels are realizable by hypotheses in $\cH$). For ease of notation, we assume WLOG $\shattz\equiv\shatt\,$, and we arbitrarily order the points $x_1,\ldots,x_V\,$. We will index our collection  of distributions by points $\omega$ in the GV-pruned $V$-dimensional hypercube, $\Omega_V\subset\{\pm1\}^{V}$ (see Lemma~\ref{lem:vg}, \citep{Gilbert1952,Varshamov1957}). Our collection of distributions is then defined by the conditional distributions,
	\begin{align*}
		\left\{\PXYomega:\omega\in\Omega_V\,,X\mid Z \sim\Unif(\shatt)\,,Y\mid Z=z,X=x_j\sim\Bern\left(\frac{1+\omega_j\gamma_z}{2}\right)\right\}\,,
	\end{align*}
	where $\gamma_z\in[0,1]$ are a set of $K$ carefully chosen parameters we will define later. Because $\shatt$ is shattered by $\cH\,$, and by the construction of $\Omega_V\,$, for any $\omega,\omegap\in\Omega_V\,$, we have,
	\begin{align*}
		\dsep{\PTomega}{\PTomegap} = \frac{H(\omega,\omegap)}{V}\sum_{z\in[K]}\qT(z)\gamma_z \geq \frac{1}{16}\sum_{z\in[K]}\qT(z)\gamma_z\,,
	\end{align*}
	where $H(\omega,\omegap)=\sum_{j\in[d]}\Ind{\omega_j\neq\omegap_j}$ is the Hamming distance. Then, using the additivity of the KL-divergence over product distributions and properties of the KL-divergence between Bernoulli distributions with parameters $\nicefrac{1}{2}\pm\nicefrac{\gamma}{2}$ we can also compute, for some absolute constant $C_{\KL}\,$,
	\begin{align*}
		\KL(\PSomega\miid\PSomegap) = \vone^\top\vn\sum_{z\in[K]}\qavg(z)\KL\left(\PXYomega\miid\PXYomegap\right) \leq C_{\KL}\vone^\top\vn\sum_{z\in[K]}\qavg(z)\gamma_z^2
	\end{align*}
	Finally, we come to the choice of $\gamma_z$ to induce the desired behavior. We can see that if we choose $\gamma_z=C_\gamma\sqrt{\frac{V\qT(z)}{\vone^\top\vn\qavg(z)}}$ for some sufficiently small absolute constant $C_\gamma\,$, then we will satisfy $\KL(\PSomega\miid\PSomegap)\leq\nicefrac{V\log(2)}{32}\leq\nicefrac{\log(\abs{\Omega_V})}{4}$ by the construction of $\Omega_V\,$. Thus, by Lemma~\ref{lem:fano}, we have,
	\begin{align*}
		\mmriskPR(B,\cH) \geq C\inf_{\vc^\top\vn\leq B}\sum_{z\in[K]}\qT(z)\sqrt{\frac{V\qT(z)}{\vone^\top\vn\qavg(z)}} \geq C\inf_{\vc^\top\vn\leq B}\sqrt{\frac{V\qmin\avgweight}{\vone^\top\vn}}
	\end{align*}
	Finally, by definition, $\nstarT$ minimizes this term and exhausts the budget, proving the statement.
\end{proof}

\section{Conclusion}
\label{sec:conclusion}

We formalized the problem of data collection from multiple heterogeneous sources when we wish to study a target population.
We showed that maximizing the effective sample size under budget constraints yields minimax-optimal policies for estimating both means and group-conditional means, and provided evidence that this principle extends to general prediction problems.
Open questions include  closing the $\sqrt{K/\qmin}$ gap between upper and lower bounds for binary classification,
and
establishing lower bounds for other prediction problems.

\section*{Acknowledgements}
The authors would like to thank 
Sameer Deshpande for helpful feedback provided during a presentation of an early version of this work and 
Michael A.\ Newton for his role as a co-advisor to Michael O.\ Harding through the NSF's TRIPODS Program for the Institute for Foundations of Data Science.
This reseach was supported by NSF Awards IIS-2441796 and DMS-2023239.

\newpage

\bibliography{bib_multisite}

\newpage

\appendix
\section{Impossibility Result}
\label{sec:pfs-imposs}
Here we provide an impossibility result for our setting that demonstrates the necessity of our bounded mean assumption.
\begin{theorem}
	\label{thm:imposs}
	Fix $B>0\,$, and consider the class of normal conditional distributions which has bounded variance but not necessarily bounded mean,
	\begin{align*}
		\infdist\defeq\{\PYZ\in\mathscr{P}(\R):\Var(Y\mid Z)\leq\sigma^2\} \supset \meandist
	\end{align*}
	Under this class of conditional distributions, the problem is hopeless; that is, for \emph{any} admissible policy $(\vn,\theta)\,$,
	\begin{align*}
		\sup_{\PYZ\in\infdist}\riskPM((\vn,\theta),\PYZ) \geq \infty\,, \qquad
		\sup_{\PYZ\in\infdist}\riskGM((\vn,\theta),\PYZ) \geq \infty
	\end{align*}
\end{theorem}
\begin{proof}
	Consider normal conditional distributions $\Pmu\in\infdist\,$, which have conditional variance $\sigma^2$ and are indexed by their mean vector $\vmu\in\R^K\,$. If we let $\Pi_\tau = \Norm(\vzero,\tau^2 I)$ be a normal prior over $\vmu\,$, then by normal-normal conjugacy, we have,
	\begin{align*}
		\vmu\mid D\sim \Norm(\tilde{\vmu},\tilde{\tau}^2I)\,,
	\end{align*}
	where,
	\begin{align*}
		\tilde{\muz} = \frac{\nz\sigma^{-2}}{\nz\sigma^{-2}+\tau^{-2}}\ybarz\,, \quad \tilde{\tau}^2 = (\nz\sigma^{-2}+\tau^{-2})^{-1}
	\end{align*}
	Thus, using the same argument as Lemma~\ref{lem:mm-to-bayes}, we have,
	\begin{align*}
		\sup_{\PYZ\in\infdist}\riskPM((\vn,\theta),\PYZ) & \geq \sum_{z\in[K]}\qT^2(z)\E_{D\sim\PSjointavg}\left[(\nz\sigma^{-2}+\tau^{-2})^{-1}\right] \\
		& \geq \sum_{z\in[K]}\qT^2(z)\tau^2\Prob_{\PSjointavg}(\nz=0)\,,
	\end{align*}
	and likewise,
	\begin{align*}
		\sup_{\PYZ\in\infdist}\riskGM((\vn,\theta),\PYZ) & \geq \sum_{z\in[K]}\tau^2\Prob_{\PSjointavg}(\nz=0)
	\end{align*}
	Finally, because this holds for any choice of $\tau>0\,$, and $\Prob_{\PSjointavg}(\nz=0)>0$ for any choice of $\vn\,$, we have,
	\begin{align*}
		\sup_{\PYZ\in\infdist}\riskPM((\vn,\theta),\PYZ) & \geq \sup_{\tau>0}\left(\sum_{z\in[K]}\qT^2(z)\tau^2\Prob_{\PSjointavg}(\nz=0)\right) = \infty\\
		\sup_{\PYZ\in\infdist}\riskGM((\vn,\theta),\PYZ) & \geq\sup_{\tau>0}\left(\sum_{z\in[K]}\tau^2\Prob_{\PSjointavg}(\nz=0)\right) = \infty
	\end{align*}
\end{proof}
\section{Proof of Theorem~\ref{thm:me-lb}}
\label{sec:pfs-me-lb}
\melb*
\begin{proof}
	We follow a classical approach of lower bounding the worst-case risk over $\meandist$ by the expected risk via a prior over a smaller subclass of distributions. Typically, this is done via selecting a suitable set of distributions, such as Normal or Bernoulli distributions, and placing a prior distribution on the parameter of interest. Computations are then made easy by choosing a conjugate prior distribution. 
	
	In our case, however, to achieve the correct dependence on $\sigma^2\,$, we wish to consider a class of distributions with normal conditional distributions $\PYZ\,$, but we cannot use a conjugate normal prior for the group-conditional means, as this distribution would place mass on distributions with means outside of $[-R,R]$ (see Theorem~\ref{thm:imposs}). Thus, we instead consider a uniform prior over $[-R,R]$ for the group-conditional means, but this yields significant technical challenges, as we no longer have a normal posterior distribution for the means. We will see, however, that our posterior distribution is a truncated normal, and with much careful work, we can achieve the desired leading term for our lower bound. 
	
	We now formalize our approach. Consider the following subclass of conditional distributions:
	\begin{align*}
		\normdist\defeq\left\{\Pmu\in\meandist:\vmu\in[-R,R]^K\,,Y\mid Z=z\sim\Norm(\muz,\sigma^2)\right\}
	\end{align*}
	This is the class of normal conditional distributions with bounded means and variances equal to $\sigma^2\,$, indexed by the mean vector $\vmu\,$. We will consider the independent joint-uniform prior over this vector, $\Pi=\Unif([-R,R]^K)\,$. We begin with the following technical lemma, lower bounding the worst-case risk by the expected risk of the Bayes estimator.
	\mmtobayes*
	\begin{proof}
		Consider the well-known result that the posterior mean minimizes the posterior $\ell_2^2$- or squared-loss from the mean, for estimating a vector and scalar, respectively. Denote by $\postmean$ the posterior mean of $\vmu\,$, conditioned on the dataset $D\,$. Then, beginning with the vector of group means, we have,
		\begin{align}
			\label{eq:bayes-lb}
			\begin{split}
				\mmriskGM(B,\meandist)
				& = \inf_{\vc^\top\vn\leq B}\inf_{\thetahatgm}\sup_{\PYZ\in\meandist}\E_{D\sim\PSjoint}\left[\norm{\thetahatGM-\theta(\PYZ)}_2^2\right] \\
				& \geq \inf_{\vc^\top\vn\leq B}\inf_{\thetahatgm}\sup_{\PYZ\in\normdist}\E_{D\sim\PSjoint}\left[\norm{\thetahatGM-\theta(\PYZ)}_2^2\right] \\
				& \geq \inf_{\vc^\top\vn\leq B}\inf_{\thetahatgm}\E_{\vmu\sim\Pi}\left[\E_{D\sim\PSjoint}\left[\norm{\thetahatGM-\vmu}_2^2\Big\vert\vmu\right]\right] \\
				& = \inf_{\vc^\top\vn\leq B}\E_{D\sim\PSjointavg}\left[\E_{\vmu\sim\Pi\mid D}\left[\norm{\postmean-\vmu}_2^2\Big\vert D\right]\right] \\
				& = \inf_{\vc^\top\vn\leq B}\sum_{z\in[K]}\E_{D\sim\PSjointavg}\left[\Var(\mu_z\mid D)\right]\,,
			\end{split}
		\end{align}
		where the final line is due to the fact that the expected $\ell_2^2$-norm of a vector is the sum of the second moments of the entries, and we are using the posterior mean. Then, we consider the fact that the posterior mean of a linear function of a vector is the same linear function applied to the posterior mean vector, along with the calculations in \eqref{eq:bayes-lb}, to see,
		\begin{align*}
			\mmriskPM(B,\meandist) & \geq \inf_{\vc^\top\vn\leq B}\E_{D\sim\PSjointavg}\left[\E_{\vmu\sim\Pi\mid D}\left[\left([\postmean-\vmu]^\top\{\qT(z)\}_{z\in[K]}\right)^2\;\Big\vert\; D\right]\right] \\
			& = \inf_{\vc^\top\vn\leq B}\sum_{z\in[K]}\qT^2(z)\E_{D\sim\PSjointavg}\left[\Var(\mu_z\mid D)\right]
		\end{align*}
	\end{proof}
	We will proceed by lower bounding this expectation, and worry about taking the infimum over sampling plans afterwards. We must now understand the posterior distribution of $\vmu\,$. We do so by leveraging the fact that the posterior distribution will be proportional to the joint distribution in the terms depending on $\vmu\,$, up to appropriate normalizing constants to achieve a proper probability distribution. Through the following proportional computations, we will see that the posterior distribution for $\vmu$ will be a jointly independent truncated normal distribution. 
	\begin{align}
		\label{eq:cond-var}
		\begin{split}
			\Pi(\vmu\mid D)
			& \propto \Pi(\vmu)\PSjoint(D\mid\vmu) \\
			& \propto \1_{[-R,R]^K}(\vmu)\exp\left(-\frac{1}{2\sigma^2}\sum_{z\in[K]}\sum_{i=1}^{\vone^\top\vn}(Y_i-\mu_z)^2\1_{z}(Z_i)\right) \\
			& \propto \prod_{z\in[K]}\1_{[-R,R]}(\mu_z)\exp\left(-\frac{\nz}{2\sigma^2}\left[\mu_z^2 - 2\mu_z\ybarz\right]\right)
		\end{split}
	\end{align}
	Using these computations, we can clearly see that the posterior distribution is a product distribution, and so the $\mu_z$ remain independent a posteriori. Further, we can recognize this form of the probability density for $\mu_z$ as belonging to a truncated normal distribution with location parameter $\ybarz\,$, scale parameter $\nicefrac{\sigma^2}{\nz}\,$, and support $[-R,R]\,$. From \cite{Johnson1994Univariate}, for ease of notation, letting $a=-\frac{\sqrt{\nz}(R+\ybarz)}{\sigma}\,$, $b=\frac{\sqrt{\nz}(R-\ybarz)}{\sigma}\,$, and $\phi$ and $\Phi$ be the standard normal pdf and cdf, respectively, we can then write the conditional variance of $\mu_z$ as,
	\begin{align*}
		\Var(\mu_z\mid D) = \frac{\sigma^2}{\nz}\left(1-\frac{b\phi(b)-a\phi(a)}{\Phi(b)-\Phi(a)} - \left(\frac{\phi(b)-\phi(a)}{\Phi(b)-\Phi(a)}\right)^2\right)
	\end{align*}
	
	Our next step is to then understand the unconditional distribution of the dataset in order to take the expectation of this quantity. We will see that we can do so explicitly by integrating $\vmu$ out from the joint density. We will demonstrate the explicit calculation for a single $z\,$, as the independence across $z$'s means we simply need to repeat this same computation $K$ times.
	\begin{align}
		\label{eq:uncond-dist}
		\begin{split}
			\PSjointavg(\ybarz\mid\nz) 
			& = \int_{\R}\PSjoint(\ybarz\mid\mu_z)\pi(\mu_z)\,d\mu_z \\
			& = \frac{1}{2R}\int_{-R}^R\sqrt{\frac{n_z}{2\pi\sigma^2}}\exp\left(-\frac{\nz}{2\sigma^2}(\ybarz-\mu_z)^2\right)\,d\mu_z \\
			& = \frac{1}{2R}\int_{-\frac{\sqrt{\nz}(R+\ybarz)}{\sigma}}^{\frac{\sqrt{\nz}(R-\ybarz)}{\sigma}}\phi(x)\,dx \\
			& = \frac{1}{2R}\left(\Phi\left(\frac{\sqrt{\nz}(R-\ybarz)}{\sigma}\right)-\Phi\left(-\frac{\sqrt{\nz}(R+\ybarz)}{\sigma}\right)\right) \\
			& = \frac{1}{2R}(\Phi(b)-\Phi(a))\,,
		\end{split}
	\end{align}
	borrowing our notation from earlier. This same term appearing as the density function of the unconditional distribution over the dataset provides some important cancellation in our computations to come. Thus, combining \eqref{eq:cond-var} and \eqref{eq:uncond-dist}, taking a change of variables $x=-\frac{\sqrt{\nz}y}{\sigma}\,$, and using the fact that $\int_{\R}x\phi(x)\,dx = 0\,$, we have,
	\begin{align}
		\label{eq:var-expr}
        \begin{split}
		\E_{D\sim\PSjointavg}&\left[\Var_{\vmu\sim\postpi}(\muz\mid D)\mid\nz\right] \\
		&\qquad\qquad\qquad
		= \frac{\sigma^2}{\nz}\left(1- \frac{\sigma}{2\sqrt{\nz}R}\int_{\R}\frac{\left[\phi\left(x+\frac{\sqrt{\nz}R}{\sigma}\right)-\phi\left(x-\frac{\sqrt{\nz}R}{\sigma}\right)\right]^2}{\Phi\left(x+\frac{\sqrt{\nz}R}{\sigma}\right) - \Phi\left(x-\frac{\sqrt{\nz}R}{\sigma}\right)}\,dx\right)\,,
        \end{split}
	\end{align}
	Appropriately lower bounding this expression is the key technical challenge of this proof. Our results for upper bounding the integral are summarized in the following technical lemma.
	\integral*
	\begin{proof}
		To prove this statement, we first recognize that this function is even, allowing us to integrate over the positive half of the real line. Then, we carefully break the positive reals into three regions, where we employ different techniques specific to each one, to upper bound the integral. Our three regions of interest are as follows: $[0,\max\{C-3,0\}), [\max\{C-3,0\},C+3), [C+3,\infty)\,$, where the first region need not be considered if $C<3\,$. We choose these values due to the known property about the standard normal distribution that more that 99.7\% of its mass is contained within $[-3,3]\,$, and the numeric stability of computing values of the normal pdf and cdf near 0. For the first region, $[-3,3]\subset[x-C,x+C]\,$, and thus the denominator can be lower bounded by 0.997, and we can analytically compute the integral in terms of $\Phi\,$. For the second region, we observe via numerical computation of the function that it is no greater than $\nicefrac{1}{2}\,$, and we generously upper bound the function by this bound. Finally, for the latter region, we utilize a technique from \cite{Vershynin2018HDP} to lower bound the denominator in terms of the numerator.
		
		We begin with the first term, in the case that $C\geq 3$ and this term is non-zero. As previously stated, we use the known lower bound for the denominator in this case of 0.997 to show,
		\begin{align*}
			I_1 & \defeq \int_0^{C-3}\frac{(\phi(x+C)-\phi(x-C))^2}{\Phi(x+C)-\Phi(x-C)}\,dx \\
			& \leq \frac{1}{0.997}\int_0^{C-3}\phi(x+C)^2 + \phi(x-C)^2\,dx \\
			& = \frac{1}{2(0.997)\sqrt{\pi}}\int_0^{C-3}\phi(\sqrt{2}(x+C)) + \phi(\sqrt{2}(x-C))\,dx \\
			& = \frac{1}{1.994\sqrt{\pi}}\left(\Phi(\sqrt{2}(2C-3))+\Phi(-3\sqrt{2})-\Phi(\sqrt{2}C)-\Phi(-\sqrt{2}C)\right) \\
			& \leq \frac{1}{1.994\sqrt{\pi}}
		\end{align*}
		Mathematically, we find the second region the most complex to handle, and there are not suitable tools for deriving an analytic expression for the integral. However, it is unnecessary for our final bounds to be more precise in this region than to simply upper bound the function by a constant and integrate over the entire region. Because the integrand is no larger than $\nicefrac{1}{2}$ for all $x\,$, we have,
		\begin{align*}
			I_2 \defeq \int_{\max\{C-3,0\}}^{C+3}\frac{(\phi(x+C)-\phi(x-C))^2}{\Phi(x+C)-\Phi(x-C)}\,dx \leq 
			\int_{\max\{C-3,0\}}^{C+3}\frac{1}{2}\,dx \leq 3
		\end{align*}
		Finally, we cover the remaining region. Consider that $1-\nicefrac{3}{t^{4}}\leq 1$ for all $t\,$, and thus we can write,
		\begin{align*}
			\Phi(x+C) - \Phi(x-C)
			& = \frac{1}{\sqrt{2\pi}}\int_{x-C}^{x+C}\exp\left(-\frac{t^2}{2}\right)\,dt \\
			& \geq \frac{1}{\sqrt{2\pi}}\int_{x-C}^{x+C}\left(1-\frac{3}{t^4}\right)\exp\left(-\frac{t^2}{2}\right)\,dt \\
			& = \left(\frac{1}{x-C}-\frac{1}{(x-C)^3}\right)\phi(x-C) - \left(\frac{1}{x+C}-\frac{1}{(x+C)^3}\right)\phi(x+C) \\
			& \geq \frac{1}{2}\left(\frac{(x-C)^2-1}{(x-C)^3}+\frac{(x+C)^2-1}{(x+C)^3}\right)\left[\phi(x-C)-\phi(x+C)\right]
		\end{align*}
		Then, using the fact that we are specifically using this bound when $x>C+3\,$, we can additionally employ the bound,
		\begin{align*}
			\frac{(x-C)^2-1}{(x-C)^3}+\frac{(x+C)^2-1}{(x+C)^3} \geq \frac{1}{x-C+1}\,,
		\end{align*}
		which holds for all $x>C+3\,$. We can now utilize this result to bound the final region of our integral,
		\begin{align*}
			I_3 & \defeq \int_{C+3}^\infty\frac{(\phi(x+C)-\phi(x-C))^2}{\Phi(x+C)-\Phi(x-C)}\,dx \\
			& \leq 2\int_{C+3}^\infty(x-C+1)(\phi(x-C)-\phi(x+C))\,dx \\
			& = 2\int_{C+3}^\infty (x-C + 1)\phi(x-C) + (2C - 1 - x+C)\phi(x+C)\,dx \\
			& = 2\left(\frac{1}{\sqrt{2\pi}}\left[\exp\left(-\frac{9}{2}\right)-\exp\left(-\frac{(2C+3)^2}{2}\right)\right] + (2C-1)\Phi(-2C-3) + \Phi(-3)\right) \\
			& \leq 0.02
		\end{align*}
		Taking these results as a whole, we get,
		\begin{align*}
			\int_{\R}\frac{(\phi(x+C)-\phi(x-C))^2}{\Phi(x+C)-\Phi(x-C)}\,dx = 2(I_1+I_2+I_3) \leq 8
		\end{align*}
	\end{proof}
	Now we utilize this result to complete our lower bound for the minimax risk. Plugging the results from Lemma~\ref{lem:integral} into the expression in \eqref{eq:var-expr} and combining with the statement from Lemma~\ref{lem:mm-to-bayes}, for the vector of group means, we have,
	\begin{align}
		\label{eq:tnorm-var-lb}
		\begin{split}
		\mmriskGM(B,\meandist)
		& \geq \inf_{\vc^\top\vn\leq B}\sum_{z\in[K]}\E_{D\sim\PSjointavg}[\Var(\mu_z\mid D)] \\
		& = \inf_{\vc^\top\vn\leq B}\sum_{z\in[K]}\E_{\nz}[\E_{D\sim\PSjointavg}[\Var(\mu_z\mid D)]\mid\nz] \\
		& \geq \inf_{\vc^\top\vn\leq B}\sum_{z\in[K]}\E_{\nz}\left[\frac{\sigma^2}{\nz}\left(1-\frac{4\sigma}{\sqrt{n_z}R}\right)\right] \\
		& \geq \inf_{\vc^\top\vn\leq B}\sum_{z\in[K]}\left(\frac{\sigma^2}{\E[\nz]} - \frac{4\sigma^3}{R(\E[\nz])^{\nicefrac{3}{2}}}\right) \\
		& = \inf_{\vc^\top\vn\leq B}\sum_{z\in[K]}\left(\frac{\sigma^2}{\vone^\top\vn\qavgbeta[\vn](z)} - \frac{4\sigma^3}{R(\vone^\top\vn\qavgbeta[\vn](z))^{\nicefrac{3}{2}}}\right) \\
		& = \inf_{\vc^\top\vn\leq B}\frac{K^2\sigma^2\avgweightUbeta[\vn]}{\vone^\top\vn} - \frac{4\sigma^3}{R(\vone^\top\vn)^{\nicefrac{3}{2}}}\sum_{z\in[K]}\frac{1}{\qavgbeta[\vn]^{\nicefrac{3}{2}}(z)}\,,
	\end{split}
	\end{align}
	and likewise for the population mean, we have,
	\begin{align*}
		\mmriskPM(B,\meandist) \geq \inf_{\vc^\top\vn\leq B}\frac{\sigma^2\avgweightbeta[\vn]}{\vone^\top\vn} - \frac{4\sigma^3}{R(\vone^\top\vn)^{\nicefrac{3}{2}}}\sum_{z\in[K]}\frac{\qT^2(z)}{\qavgbeta[\vn]^{\nicefrac{3}{2}}(z)}
	\end{align*}
	This only leaves taking the infimum over sampling plans. As noted previously, an optimal sampling plan will exhaust the entire budget, and so we can replace $\vone^\top\vn$ by $\nicefrac{B}{\avgcost(\vn)}\,$, resulting in,
	\begin{align*}
		\mmriskGM(B,\meandist) \geq \inf_{\vc^\top\vn\leq B}\frac{K^2\sigma^2\avgcost(\vn)\avgweightUbeta[\vn]}{B} - \frac{\sigma^3\avgcost(\vn)^{\nicefrac{3}{2}}}{RB^{\nicefrac{3}{2}}}\sum_{z\in[K]}(\qavgbeta[\vn](z))^{\nicefrac{3}{2}}
	\end{align*}
	Finally, we remark that choosing $\vn$ to minimize the first term, which is in fact $\nstarU\,$, can increase the entire bound away from the optimal choice by no more than $\bigO(B^{-\nicefrac{3}{2}})\,$, and thus,
	\begin{align*}
		\mmriskGM(B,\meandist) \geq \frac{K^2\sigma^2\avgcost(\nstarU)\avgweightUbeta[\nstarU]}{B} - \bigO\left(\frac{1}{B^{\nicefrac{3}{2}}}\right)\,,
	\end{align*}
	as desired. By the same arguments regarding the choice of $\vn\,$, yields,
	\begin{align*}
		\mmriskPM(B,\meandist) \geq \frac{\sigma^2\avgcost(\nstarT)\avgweightbeta[\nstarT]}{B} - \bigO\left(\frac{1}{B^{\nicefrac{3}{2}}}\right)
	\end{align*}
	
\end{proof}

\section{Proof of Theorem~\ref{thm:me-est-ub}}
\label{sec:pfs-me-ubs}
\meestub*
\begin{proof}
	As opposed to classical analyses of post-stratified estimators, our setting requires us to bound the performance of estimators \textit{unconditional} on the observed group counts $\{\nz\}_{z\in[K]}\,$, requiring specific analysis of the behavior on the events $\{\nz=0\}\,$. We begin with the simpler case of estimating the vector of group means, then tackle the additional challenges presented when estimating the population mean. For ease of notation, let $\mu_z=\E[Y\mid Z=z]\,$. First, we condition on the $\nz$'s and apply iterated expectation to bound the risk in terms of $\sigma^2$ and $\nz\,$,
	\begin{align}
		\label{eq:me-ub-pf}
		\begin{split}
			\riskGM((\vn,\thetahatVM),\PYZ)
			& = \E_{D\sim\PSjoint}\left[\norm{\thetahatVM-\theta(\PT)}_2^2\right] \\
			& = \sum_{z\in[K]}\E_{\nz}\left[\E_{D\sim\PSjoint}\left[(\ybarz-\mu_z)^2\mid\nz\right]\right] \\
			& \leq \sum_{z\in[K]}\E_{\nz}\left[\frac{\sigma^2}{\nz}\1_{(0,\infty)}(\nz) + \mu_z^2\1_{\{0\}}(\nz)\right] \\
			& \leq \sum_{z\in[K]}\left(\E_{\nz}\left[\frac{\sigma^2}{\nz}\1_{(0,\infty)}(\nz)\right] + R^2(1-\qavg(z))^{\vone^\top\vn}\right)
		\end{split}
	\end{align}
	With $R$ bounded, the second term decays exponentially in $\vone^\top\vn\,$, and so incurs $\littleO((\vone^\top\vn)^{-1})$ risk.
	It remains to bound the $\nicefrac{\sigma^2}{\nz}$ term. We do so by studying the Taylor expansion of $\nz^{-1}\1_{(0,\infty)(\nz)}$ about $\E[\nz]=\vone^\top\vn\qavg(z)\,$,
	\begin{align*}
		\E\left[\frac{\1_{(0,\infty)}(\nz)}{\nz}\right]
		& = \frac{1}{\vone^\top\vn\qavg(z)} + \sum_{k=1}^\infty\E\left[\frac{(-1)^k(\nz-\vone^\top\vn\qavg(z))^k}{(\vone^\top\vn\qavg(z))^{k+1}}\right] \\
		& = \frac{1}{\vone^\top\vn\qavg(z)} + \sum_{k=1}^\infty\frac{1}{(\vone^\top\vn)^{2k}}\E\left[\frac{(\qhat(z)-\qavg(z))^{2k}}{\qavg^{2k+1}(z)}\right] \\
		& = \frac{1}{\vone^\top\vn\qavg(z)} + \littleO\left(\frac{1}{\vone^\top\vn}\right)\,,
	\end{align*}
	where $\qhat(z) = \nicefrac{\nz}{\vone^\top\vn}$ to make it easier to see the explicit scaling in $(\vone^\top\vn)^{-2k}$ for the remaining terms. This proves the statement for $\thetahatVM$ by using the fact that $\sum_{z\in[K]}\qavg^{-1}(z) = K^2\avgweightU\,$.
	
	For the statement for $\thetahatPS\,$, we are able to reuse nearly all of this work. The work in \eqref{eq:me-ub-pf} applies here as well, but there are additional cross terms on the events that pairs of $\nz$'s are zero,
	\begin{align*}
		\riskPM((\vn,\thetahatPS),\PYZ) & \leq \sum_{z\in[K]}\qT^2(z)\left(\E_{\nz}\left[\frac{\sigma^2}{\nz}\1_{(0,\infty)}(\nz)\right] + R^2(1-\qavg(z))^{\vone^\top\vn}\right) \\
		& \qquad + \sum_{z\neq z^\prime}\qT(z)\qT(z^\prime)R^2(1-\qavg(z)-\qavg(z^\prime))^{\vone^\top\vn}\,,
	\end{align*}
	which only incur additional $\littleO((\vone^\top\vn)^{-1})$ risk. Applying the same Taylor expansion concludes the proof.
\end{proof}

\section{Proof of Theorem~\ref{thm:avg-risk-ub}}
\label{sec:pfs-learning-prediction}
\predub*
\begin{proof}[Proof of Theorem \ref{thm:avg-risk-ub}]
	We begin by studying IWERM under an arbitrary sampling plan $\vn\,$. First, we prove that the importance weighted empirical risk is, under the source distribution, an unbiased estimator of the population loss on the target distribution:
	\begin{equation*}
	\begin{aligned}
		\E_{\PSavg}\left[\frac{\qT(Z)}{\qavg(Z)}\ell(h(Z,X),Y)\right]
		& = \sum_{z=1}^K\E_{\PSavg}\left[\frac{\qT(Z)}{\qavg(Z)}\ell(h(Z,X),Y)\mid Z=z\right]\qavg(z) \\
		& = \sum_{z=1}^K\E_{\PSavg}\left[\ell(h(z,X),Y)\mid Z=z\right]\qavg(z)\frac{\qT(z)}{\qavg(z)} \\
		& = \sum_{z=1}^K\E_{\PSavg}\left[\ell(h(z,X),Y)\mid Z=z\right]\qT(z) \\
		& = \E_{\PT}[\ell(h(Z,X),Y)]
	\end{aligned}
	\end{equation*}
	Then, we can use this fact combined with the IWERM procedure to bound the excess population loss of the IWERM estimator in terms of the worst-case generalization error of the class $\cH\,$. First, consider that, because the range of $\ell$ is compact, for any $\epsilon>0$ there exists some $h_\epsilon\in\cH$ such that $\E_{\PT}[\ell(h_\epsilon(Z,X),Y)] \leq \inf_{h\in\cH}\E_{\PT}[\ell(h(Z,X),Y)] + \epsilon\,$. Further, let us denote the empirical average over the data set as $\Exphat\,$. Then, we can show,
	\begin{equation}
	\label{eq:erm-gen-bnd}
	\begin{aligned}
		\loss(\hiwerm,\cH,\PT)
		& = \E_{\PT}\left[\ell(\hiwerm(Z,X),Y)\right] - \inf_{h\in\cH}\E_{\PT}\left[\ell(h(Z,X),Y)\right] \\
		& \leq \E_{\PT}\left[\ell(\hiwerm(Z,X),Y)\right] - \E_{\PT}\left[\ell(h_\epsilon(Z,X),Y)\right] + \epsilon \\
		& = \E_{\PT}\left[\ell(\hiwerm(Z,X),Y)\right] \pm \Exphat\left[\frac{\qT(Z)}{\qavg(Z)}\ell(h_\epsilon(Z,X),Y)\right] \\
        & \qquad - \E_{\PT}\left[\ell(h_\epsilon(Z,X),Y)\right] + \epsilon \\
		& \leq \left(\E_{\PT}\left[\ell(\hiwerm(Z,X),Y)\right] - 
			\Exphat\left[\frac{\qT(Z)}{\qavg(Z)}\ell(\hiwerm(Z,X),Y)\right]\right) \\
		& \qquad + \left(\Exphat\left[\frac{\qT(Z)}{\qavg(Z)}\ell(h_\epsilon(Z,X),Y)\right] - 
			\E_{\PT}\left[\ell(h_\epsilon(Z,X),Y)\right]\right) + \epsilon \\
		& \leq 2\sup_{h\in\cH}\left(\frac{1}{\vone^\top\vn}\sum_{i=1}^{\vone^\top\vn}\left(\frac{\qT(z_i)}{\qavg(z_i)}\ell(h(z_i,x_i),y_i) -
			\E_{\PSavg}\left[\frac{\qT(Z)}{\qavg(Z)}\ell(h(Z,X),Y)\right]\right)\right)\,,
	\end{aligned}
	\end{equation}
	Where we drop $\epsilon$ in the final line because $\epsilon>0$ was arbitrary and we can take $\epsilon\searrow0$ without affecting the rest of the statement. Now, we can let $w_i = (z_i,x_i,y_i)\,$, and consider the function class,
	\begin{equation*}
		\scrF = \left\{f:\cZ\times\cX\times\cY\to\R_+:f((z,x,y)) = \frac{\qT(z)}{\qavg(z)}\ell(h(z,x),y)\,, h\in\cH\right\}\,,
	\end{equation*}
    To study this function class, and thus bound the excess risk of our proposed policy, we introduce the following definitions and technical results from \cite{Baraud2016Boundingexpectation}.
    \begin{definition}[Definition 2.1 of~\citet{Baraud2016Boundingexpectation}]
	A class $\cC$ of subsets of some set $\cZ$ is said to \textbf{shatter} a finite subset $Z$ of $\cZ$ if $\{C\cap Z:C\in\cC\}=\mathscr{P}(Z)$ or, equivalently, $\abs{\{C\cap Z:C\in\cC\}}=2^{\abs{Z}}\,$. A non-empty class $\cC$ of subsets of $\cZ$ is a \textbf{VC-class} if there exists an integer $k\in\N$ such that $\cC$ cannot shatter any subset of $\cZ$ with cardinality larger than $k\,$. The \textbf{dimension} $d\in\N$ of $\cC$ is then the smallest of these integers $k\,$.
    \end{definition}
    \begin{definition}[Definition 2.2 of~\citet{Baraud2016Boundingexpectation}]
	Let $\scrF$ be a non-empty class of functions on a set $\cX\,$. We shall say that $\scrF$ is \textbf{weak VC-major with dimension $d\in\N$} if $d$ is the smallest integer $k\in\N$ such that, for all $u\in\R\,$, the class,
	\begin{equation*}
		\cC_u(\scrF) = \left\{\left\{x\in\cX:f(x) > u\right\} : f\in\scrF\right\}
	\end{equation*}
	is a VC-class of subsets of $\cX$ with dimension not larger than $k\,$.
    \end{definition}
    \begin{lemma}[Proposition 2.3 of~\citet{Baraud2016Boundingexpectation}]
	\label{lem:vc-monotone}
	Let $\scrF$ be weak VC-major with dimension $d$. Then for any monotone function $F\,,F\circ\scrF=\{F\circ f:f\in\scrF\}$ is weak VC-major with dimension not larger than $d\,$.
    \end{lemma}
    \begin{lemma}[Corollary 2.1 of~\citet{Baraud2016Boundingexpectation}]
	\label{lem:ep-exp}
	Let $X_1,\ldots,X_n$ be \iid random variables following any arbitrary distribution. Let $\scrF$ be a weak VC-major class with dimension not larger than $d\geq 1$ consisting of functions with values in $[-b,b]$ for some $b>0\,$, and define,
	\begin{equation*}
		\sigma^2\defeq \sup_{f\in\cF}\frac{1}{n}\sum_{i=1}^n\E[f^2(X_i)]\,, \quad Z_n(\scrF)\defeq \sup_{f\in\cF}\abs{\frac{1}{n}\sum_{i=1}^n(f(X_i)-\E[f(X_i)])}
	\end{equation*}
	Then,
	\begin{equation}
		\E\left[Z_n(\scrF)\right] \leq \sigma\log\left(\frac{eb}{\sigma}\right)\sqrt{\frac{32d\log(2end^{-1})}{n}} + \frac{16bd\log(2end^{-1})}{n}
	\end{equation}
    \end{lemma}
	Using the notation of Lemma~\ref{lem:ep-exp} and the results of \eqref{eq:erm-gen-bnd} we can write,
	\begin{equation}
		\label{eq:iwerm-ep}
		\lossavg(\hiwerm,\cH,\PT) \leq 2Z_n(\scrF)
	\end{equation}
	It then remains to study the properties of $\scrF$ and to apply the results of Lemmas~\ref{lem:ep-exp} as appropriate. First, we clearly have $f(w)\in[-\maxweight,\maxweight]$ for all $f\in\scrF\,$. Studying the second moment, we can see,
	\begin{equation*}
	\begin{aligned}
		\sup_{f\in\scrF}\frac{1}{\vone^\top\vn}\sum_{i=1}^{\vone^\top\vn}\E_{D\sim\PSavg}[f^2(w_i)]
		& = \sup_{h\in\cH}\frac{1}{\vone^\top\vn}\sum_{i=1}^{\vone^\top\vn}\E_{D\sim\PSavg}\left[\left(\frac{\qT(z_i)}{\qavg(z_i)}\ell(h(z_i,x_i),y_i)\right)^2\right] \\
		& \leq \E_{\PSavg}\left[\left(\frac{\qT(Z)}{\qavg(Z)}\right)^2\right] \\
		& = \sum_{z=1}^K \qavg(z)\left(\frac{\qT(z)}{\qavg(z)}\right)^2 \\
		& = \sum_{z=1}^K \qT(z)\frac{\qT(z)}{\qavg(z)} \\
		& = \avgweight
	\end{aligned}
	\end{equation*}
	Finally, we must determine the complexity of the class $\scrF\,$. We wish to determine, for all $u\in\R$, if the class,
	\begin{equation*}
		\cC_u(\scrF) = \left\{\left\{w\in\cZ\times\cX\times\cY:f(w)>u\right\}:f\in\scrF\right\}
	\end{equation*}
	is a VC-class of subsets of $\cZ\times\cX\times\cY$, and if so, its dimension. First, we can consider a disjoint partition by the value of $z$, studying the classes,
	\begin{equation*}
		\cC_u^{(z)}(\scrF) = \left\{\{z\}\times\{(x,y)\in\cX\times\cY:f(z,x,y)>u\}:f\in\scrF\right\}
	\end{equation*}
	and utilizing the fact that if each of these collections are a VC-class with dimensions $d_z$, then their disjoint union is also a VC-class with dimension at most the sum of the $d_z$'s.\footnote{%
	Consider that for each $2^{d_z}$ subsets for each $z$, we can create at most another $2^{d_{z^\prime}}$ subsets with a separate $z^\prime$, making at most $2^{\sum_zd_z}$ subsets that can be created, hence the VC-dimension of the union of the collection is at most the sum}
	For each of these collections, we can study the complexity of the functions,
    \begin{align*}
        \cG_z = \{g_z:\cX\times\cY\to\R \mid g_z(x,y)=h(z,x)-y\,, h\in\cH\}\,,
    \end{align*}
    and then apply Lemma~\ref{lem:vc-monotone} to understand the complexity of $\cC_u^{(z)}$. By assumption that $\cH\,$, when restricted to any $z\,$, has finite pseudo-dimension, the collection $\cG_z$ is weak VC-major with dimension at most $\pseudodimH$. Then, we can apply Lemma~\ref{lem:vc-monotone} to show that the classes,
	\begin{equation*}
		\cG_z^+ = \{g_z\vee0:g_z\in\cG_z\}\,, \quad \cG_z^- = \{-g_z\vee0:g_z\in\cG_z\}
	\end{equation*}
	are both also weak VC-major with dimension $\pseudodimH\,$. Then, recognizing that we can write,
	\begin{equation*}
	\begin{aligned}
		\cC_u(\cG_z^{\pm}) 
		& = \{\{(x,y)\in\cX\times\cY:\abs{g(x,y)}>u\}:g\in\cG_z\} \\
		& = \{A\cup B:A\in\cC_u(\cG_z^+),B\in\cC_u(\cG_u^-)\}\,,
	\end{aligned}
	\end{equation*}
	we know that $\cC_u(\cG_z^{\pm})$ is a VC-class with dimension at most $2\pseudodimH\,$. Finally, we use the fact that $\ell(y,y^\prime)$ is monotone in $\abs{y-y^\prime}$ and that $\frac{\qT(z)}{\qavg(z)}$ is a constant for fixed $z$ to conclude that for all $u\in\R\,,\cC_u^{(z)}(\scrF)$ is a VC-class with dimension at most $2\pseudodimH\,$, meaning $\scrF$ is a weak VC-major class with dimension at most $2\pseudodimH K\,$. This allows us to take expectations on both sides of \eqref{eq:iwerm-ep} and apply the results of Lemma~\ref{lem:ep-exp} to achieve the bound,\footnote{%
	For the interested reader, we note that it is also possible to combine this result with Theorem 2.1 and Lemma 2.4 of \cite{Marchina2021Concentrationinequalities} to construct a related bound with high probability, rather than in expectation.}
    \begin{align*}
        \risk((\vn,\hiwerm),\cH,\PXY) & \leq \log\left(\frac{e\maxweight}{\sqrt{\avgweightbeta[\vn]}}\right)\sqrt{\frac{192\pseudodimH K\avgweightbeta[\vn]\log\left(\frac{2e\vone^\top\vn}{\pseudodimH K}\right)}{\vone^\top\vn}} \\
			& \qquad + \frac{64dK\maxweight\log\left(\frac{2e\vone^\top\vn}{\pseudodimH K}\right)}{\vone^\top\vn}
    \end{align*}
    Discounting the logarithmic factors, we see that the leading term in this bound is dependent on $\vn$ via $\frac{\avgweight}{\vone^\top\vn}\,$, like many of our other bounds, and thus we utilize the same sampling plan $\nstarT$ to maximize the effective sample size. This proves the desired bound, and approximately matches the lower bound in Theorem~\ref{thm:avg-risk-lb}.
\end{proof}

\section{A New Minimax Framework}
\label{sec:framework}
We use this section to construct a new framework for proving minimax lower bounds in our setting, as the existing tools do not apply here. Importantly, we point out that this framework can be readily applied to a broader class of problems beyond ours, allowing any arbitrary between source and target distributions. We begin with the following definitions.
\begin{definition}
    For a loss function $\ell:[K]\times\cX\to[0,1]\,$, a hypothesis class $\cH\,$, and a target distribution $\PT\,$, we define the \emph{excess target population loss} of a hypothesis $h\in\cH$ as,
    \begin{align*}
        \loss(h,\cH,\PT) \defeq \E_{\PT}\left[\ell(h(Z,X),Y)\right] - \inf_{h^\prime\in\cH}\E_{\PT}\left[\ell(h^\prime(Z,X),Y)\right]
    \end{align*}
\end{definition}
\begin{definition}
    For a hypothesis class $\cH$ and excess target population loss $\loss\,$, we define the \emph{separation} \wrt $\cH$ between any two target distributions $\PTomega[1],\PTomega[2]$ as,
    \begin{align*}
		\dsep{\PTomega[1]}{\PTomega[2]}\defeq\sup\{\delta\geq0:
		L(h,\cH,\PTomega[1])\leq\delta &\implies L(h,\cH,\PTomega[2])\geq\delta\;\forall\;h\in\cH\,, \\
		L(h,\cH,\PTomega[2])\leq\delta &\implies L(h,\cH,\PTomega[1])\geq\delta\;\forall\;h\in\cH\}\,.
	\end{align*}
\end{definition}
\begin{definition}
    We call a collection $\{(\PSomega[1],\PTomega[1]),\ldots,(\PSomega[N],\PTomega[N])\}$ of (source,target) distribution pairs \emph{target}-$\delta$-\emph{separated} if, for all $j\neq k\,$, we have,
    \begin{align*}
        \dsep{\PTomega[j]}{\PTomega[k]} \geq \delta
    \end{align*}
\end{definition}
We now use these definitions to construct a new version of the ``reduction-to-testing'' lemma for lower bounding the minimax risk in the setting where source and target distributions differ.
\redtotest*
\begin{proof}
	We begin by using the fact that the maximum of a restricted class is upper bounded by the supremum of a larger class and Markov's inequality to show,
	\begin{align*}
		\mmriskPR(B,\cH)
			& = \inf_{\vc^\top\vn\leq B}\inf_{\hhat}\sup_{\PXY\in\preddist}\riskPR((\vn,\hhat),\cH,\PXY) \\
			& \geq \inf_{\vc^\top\vn\leq B}\inf_{\hhat}\max_{j\in[N]}\riskPR((\vn,\hhat),\cH,\PXYomega[j]) \\
			& \geq \inf_{\vc^\top\vn\leq B}\delta_{\vn} \inf_{\hhat}\max_{j\in[N]}\PSomega[j](\loss(\hhat_D,\cH,\PTomega[j])>\delta_{\vn})
	\end{align*}
	Now, consider the test function $\psi_{\hhat}(D) = \argmin_{j\in[N]}L(\hhat_D,\cH,\PTomega[j])\,$. Then, suppose $D\sim \PSomega[j]$ but $\psi_{\hhat}(D)=k\neq j\,$. By construction, $\psi_{\hhat}(D)=k\implies L(\hhat_D,\cH,\PTomega[j])\geq\delta_{\vn}\,$, meaning,
	\begin{align*}
		\PSomega[j](\loss(\hhat_D,\cH,\PTomega[j])>\delta_{\vn}) \geq \PSomega[j](\psi_{\hhat}(D)\neq j)
	\end{align*}
	Combining this with our previous result, we have,
	\begin{align*}
		\mmriskPR(B,\cH)
			& \geq \inf_{\vc^\top\vn\leq B}\delta_{\vn} \inf_{\hhat}\max_{j\in[N]}\PSomega[j](L(\hhat_D,\cH,T_j)>\delta_{\vn}) \\
			& \geq \inf_{\vc^\top\vn\leq B}\delta_{\vn} \inf_{\hhat}\max_{j\in[N]}\PSomega[j](\psi_{\hhat}(D)\neq j) \\
			& \geq \inf_{\vc^\top\vn\leq B}\delta_{\vn} \inf_{\psi}\max_{j\in[N]}\PSomega[j](\psi(D)\neq j)
	\end{align*}
\end{proof}
We now combine this result with an application of Fano's inequality to construct a new version of ``Fano's method'' for lower bounding the minimax risk in the setting where source and target distributions differ.
\begin{restatable}[Multi-source-target Fano's method]{lemma}{fano}
		\label{lem:fano}
		Fix $\delta_{\vn}>0\,$, possibly depending on the sampling plan $\vn\,$, and a hypothesis class $\cH$. If the collection $\{(\PSomega[1],\PTomega[1]),\ldots,(\PSomega[N],\PTomega[N])\}$ is target-$\delta_{\vn}$-separated \wrt $\cH\,$, then,
		\begin{align*}
			\mmriskPR(B,\cH)\geq\inf_{\vc^\top\vn\leq B}\delta_{\vn}\left(1-\frac{N^{-2}\sum_{j,k}\KL(\PSomega[j]\miid\PSomega[k])+\log(2)}{\log(N)}\right)
		\end{align*}
		Thus, if we have $\KL(\PSomega[j]\miid\PSomega[k])\leq\nicefrac{\log(N)}{4}$ for all $j,k$ and $N\geq4\,$, then $\mmriskPR(B,\cH)\geq\inf_{\vc^\top\vn\leq B}\nicefrac{\delta_{\vn}}{4}\,$.
	\end{restatable}
\begin{proof}
	Define a random variable $V\in[N]$ \st $\Prob(V=j)=\nicefrac{1}{N}$ for $j\in[N]\,$, and conditioned on $\{V=j\}\,$, let us draw $D\sim \PSomega[j]\,$. Then, we have the joint distribution,
	\begin{align}
    \label{eq:markov}
		\Prob(D\in A,V=j) = \Prob(D\in A\mid V=j)\Prob(V=j) = \frac{1}{N}\PSomega[j](D\in A)
	\end{align}
    We now introduce the following information theoretic quantities for random variables in order to prove our result.
    \begin{definition}
        Let $X\sim P$ be a random variable on a probability space $\Omega\,$. The \emph{entropy} of $X$ is,
        \begin{align*}
            H(X) \defeq \E[-\log(P(X))]
        \end{align*}
        Let $Y$ be a second random variable defined on the same probability space, with the pair following law $Q\,$, that is $(X,Y)\sim Q$. Then, we likewise define the \emph{joint entropy} of $X$ and $Y$ and the \emph{conditional entropy} of $X$ given $Y$ as,
        \begin{align*}
            H(X,Y) \defeq \E[-\log(Q(X,Y))]\,, \quad H(X\mid Y) \defeq \E[-\log(Q(X\mid Y))]
        \end{align*}
    \end{definition}
    \begin{definition}
        Let $X,Y$ be random variables defined on a shared probability space $\Omega\,$. Define their joint distribution as $P\,$, \ie $(X,Y)\sim P\,$, and let $P_X$ and $P_Y$ be the corresponding marginal distributions. The, the \emph{mutual information} between $X$ and $Y$ is defined as,
        \begin{align*}
            I(X,Y)\defeq\KL(P\miid P_X\otimes P_Y) = \E\left[\log\left(\frac{P(X,Y)}{P_X(X)P_Y(Y)}\right)\right]
        \end{align*}
        Additionally, note that this satisfies $I(X,Y) = H(X)-H(X\mid Y) = H(Y)-H(Y\mid X)\,$.
    \end{definition}
    Then, we introduce Fano's inequality using these definitions.
    \begin{lemma}[Fano's inequality \citep{Fano1961Information}]
        Let $X\in\cX$ be a random variable such that $\abs{\cX}<\infty$ and $\abs{X}<\infty$. Let $Y\in\cY$ and $\Xhat\in\cX$ be additional random variables such that $X\to Y\to \Xhat$ forms a Markov chain. Then, letting $Z\sim\Bern(\Prob(X\neq\Xhat))\,$, we have,
        \begin{align*}
            H(X\mid Y) \leq H(X\mid\Xhat)\leq \Prob(X\neq\Xhat)\log(\abs{\cX}) + H(Z)\,,
        \end{align*}
        and thus,
        \begin{align*}
            \Prob(X\neq\Xhat)\geq\frac{H(X\mid Y)-\log(2)}{\log(\abs{\cX})}
        \end{align*}
    \end{lemma}
	Returning to our construction in \eqref{eq:markov}, for any test function $\psi\,$,  clearly $V\to D\to \psi(D)$ forms a Markov chain, so we can apply Fano's inequality to show,
	\begin{align*}
		\Prob(\psi(D)\neq V)
			& \geq \frac{H(V\mid D)-\log(2)}{\log(N)} = \frac{H(V)-I(V,D)-\log(2)}{\log(N)} = 1-\frac{I(V,D)-\log(2)}{\log(N)}\,,
	\end{align*}
	Then, we can use the definition of $I(V,D)$ and convexity of the KL-Divergence in the second argument to show that, under our construction, we have,
	\begin{align*}
		I(V,D)
			& = \E_{V,D}\left[\log\left(\frac{p(V,D)}{p(V)p(D)}\right)\right] \\
			& = \frac{1}{N}\sum_{j=1}^N\int S_j(D)\log\left(\frac{S_j(D)\Prob(V=j)}{\avgS(D)\Prob(V=j)}\right)\,dD \\
			& = \frac{1}{N}\sum_{j=1}^N\int S_j(D)\log\left(\frac{S_j(D)}{\avgS(D)}\right)\,dD \\
			& = \frac{1}{N}\sum_{j=1}^N\KL(S_j\miid\avgS) \\
			& \leq \frac{1}{N^2}\sum_{j,k}\KL(S_j\miid S_k)
	\end{align*}
	Finally, we apply the results of Lemma \ref{lem:red-test}, max$\geq$avg, and our lower bound on $\Prob(\psi(D)\neq V)$ to achieve the desired result.
\end{proof}

\section{Proof of Theorem~\ref{thm:avg-risk-lb}}
\label{sec:pfs-pred-lb}
\predlb*
\begin{proof}
	We prove this lower bound by an application of our framework developed in Appendix~\ref{sec:framework}. We begin by constructing our class of alternative distributions. Let $V=\vcdimH\,$. We use the fact that $\cH$ has VC-dimension $V$ over $\cX\,$, when restricted to any $z\in[K]\,$, to select sets $\shattz\subset\cX$ of size $V$ that are each \textit{shattered} by $\cH$ (all $2^V$ labels are realizable by hypotheses in $\cH$). For ease of notation, we assume WLOG $\shattz\equiv\shatt\,$, and we arbitrarily order the points $x_1,\ldots,x_V\,$. In order to structure our collection of distributions with desirable qualities, we introduce the following technical lemma, due to \cite{Gilbert1952} and \cite{Varshamov1957}.
    \begin{lemma}[Result due to \citet{Gilbert1952,Varshamov1957}]
	\label{lem:vg}
	Let $d\geq8\,,\Omega = \{\pm1\}^d\,$ and define the Hamming distance $H:\Omega^2\to\N$ by $H(\omega,\omega^\prime)=\sum_{i=1}^d\Ind{\omega_i\neq\omega_i^\prime}\,$. Then, there exists a subset $\Omega^\prime\subset\Omega\,$, called the `GV-pruned hypercube,' satisfying the following two properties,
	\begin{enumerate}
		\item $\abs{\Omega^\prime} \geq 2^{\nicefrac{d}{8}}$
		\item $\displaystyle\min_{\omega,\omega^\prime\in\Omega^\prime}H(\omega,\omega^\prime) \geq \frac{d}{8}$
	\end{enumerate}
    \end{lemma}
    We will index our collection  of distributions by points $\omega$ in GV-pruned hypercube of dimension $V\,$, $\Omega_d\subset\{\pm1\}^{V}$ Our collection of distributions is then defined by the conditional distributions,
	\begin{align*}
		\left\{\PXYomega:\omega\in\Omega_V\,,X\mid Z \sim\Unif(\shatt)\,,Y\mid Z=z,X=x_j\sim\Bern\left(\frac{1+\omega_j\gamma_z}{2}\right)\right\}\,,
	\end{align*}
	where $\gamma_z\in[0,1]$ are a set of $K$ parameters we will define later. Because $\shatt$ is shattered by $\cH\,$, and by the construction of $\Omega_V\,$, for any $\omega,\omegap\in\Omega_V\,$, we have,
	\begin{align*}
		\dsep{\PTomega}{\PTomegap} = \frac{H(\omega,\omegap)}{V}\sum_{z\in[K]}\qT(z)\gamma_z \geq \frac{1}{16}\sum_{z\in[K]}\qT(z)\gamma_z\,,
	\end{align*}
	where $H(\omega,\omegap)=\sum_{j\in[V]}\Ind{\omega_j\neq\omegap_j}$ is the Hamming distance. Then, using the additivity of the KL-divergence over product distributions, and the fact that the KL-divergence between Bernoulli distributions with parameters $\nicefrac{1}{2}\pm\nicefrac{\gamma}{2}$ is bounded by $C_{\KL}\gamma^2$ for an absolute constant $C_{\KL}\,$, we can also compute,
	\begin{align*}
		\KL(\PSomega\miid\PSomegap) & = \vone^\top\vn\sum_{z\in[K]}\qavg(z)\KL\left(\PXYomega\miid\PXYomegap\right) \leq C_{\KL}\vone^\top\vn \sum_{z\in[K]}\qavg(z)\gamma_z^2
	\end{align*}
	Finally, we come to the choice of $\gamma_z$ to induce the desired behavior. We can see that if we choose $\gamma_z=C_\gamma\sqrt{\frac{V\qT(z)}{\vone^\top\vn\qavg(z)}}$ for some sufficiently small absolute constant $C_\gamma\,$, then we will satisfy $\KL(\PSomega\miid\PSomegap)\leq\nicefrac{V\log(2)}{32}\leq\nicefrac{\log(\abs{\Omega_V})}{4}$ by the construction of $\Omega_V\,$. Thus, by Lemma~\ref{lem:fano}, we have,
	\begin{align*}
		\mmriskPR(B,\cH) \geq C\inf_{\vc^\top\vn\leq B}\sum_{z\in[K]}\qT(z)\sqrt{\frac{V\qT(z)}{\vone^\top\vn\qavg(z)}} \geq C\inf_{\vc^\top\vn\leq B}\sqrt{\frac{V\qmin\avgweight}{\vone^\top\vn}}
	\end{align*}
	Finally, by definition, $\nstarT$ minimizes this term and exhausts the budget, proving the statement.
\end{proof}
\section{Numerical Experiments}
\label{sec:experiments}

We conclude by including a brief suite of experiments to corroborate our theoretical findings. We find it particularly instructive to observe the degree to which our proposed sampling plan outperforms other seemingly reasonable approaches. This underscoring the need to understand the dynamics at play in this problem in order to get the most out of a data collection scheme. We construct two settings for the source distributions: one with 5 groups and 10 sources and one with 20 groups and 20 sources. 

The first setting has multiple ``sparse'' sources, with only a subset of groups available to sample, with these sources being relatively cheaper to sample. This setting is meant to specifically highlight our approach of maximizing the effective sample size by leveraging cheap samples from certain sources to craft a mixture distribution which is both cheap and close to the target. The exact distributions and costs are given in Table~\ref{tab:5-10}.
\begin{table}[!ht]
\centering
    \begin{tabular}{r|ccccc|c}
        \multicolumn{1}{c}{} & \multicolumn{5}{c}{Group} & \\
        Source & A & B & C & D & E & Cost \\
        \hline
        1 & 1 & 0 & 0 & 0 & 0 & 0.02\\
        2 & 0.05 & 0.15 & 0.15 & 0.15 & 0.5 & 3\\
        3 & 0.05 & 0.2 & 0.3 & 0.35 & 0.1 & 4\\
        4 & 0.05 & 0.3 & 0.55 & 0.1 & 0 & 3\\
        5 & 0.05 & 0.25 & 0.15 & 0 & 0.55 & 0.1\\
        6 & 0.05 & 0.05 & 0.4 & 0.45 & 0.05 & 2.4\\
        7 & 0.05 & 0.15 & 0.6 & 0.05 & 0.15 & 1.6\\
        8 & 0.05 & 0.05 & 0.05 & 0.4 & 0.45 & 2\\
        9 & 0.05 & 0.3 & 0.3 & 0.05 & 0.3 & 2\\
        10 & 0 & 0.5 & 0 & 05 & 0 & 1
    \end{tabular}
    \caption{5 group, 10 source setting distributions and costs.}
    \label{tab:5-10}
\end{table}

The second setting takes a single distribution over the groups and ``cycles'' it by one entry for each source (moving the 1st entry to the 20th and shifting the others each down one accordingly), and the the sampling costs linearly span $[0.1,1]\,$. This setting is meant to be more realistic and give alternative sampling plans more of a ``fighting chance,'' but we will see that our approach still far outperforms others. The distribution of the first source for this setting is,
\begin{align*}
    (&0.0057, 0.0307, 0.0625, 0.0938, 0.1547, 0.0392, 0.0380 , 0.1256, 0.0347, 0.0825,\\
    &0.0370 , 0.0154, 0.0379, 0.0410 , 0.0268, 0.0824, 0.0010 , 0.0313, 0.0295, 0.0303)\,,
\end{align*}

For our simulation studies, we compare to four alternative sampling plans. The first is the Uniform sampling plan, which collects the same number of samples from each source, representing a completely naive planner. The second is the Inverse-Cost sampling plan, which collects a number of samples from each source inversely proportional to the sampling cost of the source, representing a cost-focused planner. The third is the Nearest sampling plan, which finds the allocation resulting in a mixture group distribution $\qavg$ as close to the target distribution as possible in the total variation distance, representing a target-matching focused planner. The final plan, Hybrid, takes cues from both Inverse-Cost and Nearest; this plan computes the allocation that is closest in total variation distance, and then allocates proportional to those amounts divided by the sampling costs.
\insertsimfigure
\insertsimfigureincr
\insertsimfigurepyrm
We study estimating the population mean and the vector of group means under the post-stratified estimator and the vector of observed conditional means we propose in \S~\ref{sec:mean-est}. A true mean vector is initially generated randomly from a $\Norm(\vzero,10I)$ distribution, then fixed for an experimental setting. For each replication, a dataset is generated by first generating group identities $Z$ according to the source distributions, then the responses are generated from a $\Norm(\muz,5)$ distribution. Each method has access to the same simulated data, and for each simulated dataset, a range of budgets from \$25 to \$500 are considered. Each setting is replicated 100 times.

We also study a binary classification setting under the IWERM procedure proposed in \S~\ref{sec:prediction}. We consider a setting with 20 additional numeric features, following a $\Norm(\vzero,I)$ distribution, not depending on the group identity. As before, data is generated by first generating group identities $Z$ and then generating covariates $X$ and response $Y\,$, which now follows a $\Bern(\Phi(X^\top\beta_z))\,$, where the $\beta_z$'s are true coefficient vectors generated according to a $\Norm(\vzero,10I)$ distribution prior to the experimental suite. Again, each method has access to the same simulated data, a range of budgets from \$25 to \$500 are considered, and each setting is replicated 100 times.

We present the results of the first suite of experiments in Figure~\ref{fig:unif}. For this setting, we let the target distribution be $\unif\,$, to allow the same optimal sampling plan to be used across all three settings. We can clearly see that our method strongly outperforms the other approaches here.

To explore the utility of our method across a variety of target distributions, we additionally include a second set of experiments for estimating the population mean and for binary classification. In this case, we consider the same source and cost settings as before, but we now consider two different target distributions in each setting. The first is the ``increasing'' target, where the target proportions increase linearly in the order of the groups. The second is the ``pyrammid'' target, where the target proportions increase linearly until reaching the halfway mark, and then decrease linearly back down. The results for the ``increasing'' target are included in Figure~\ref{fig:incr}, and the results for the ``pyramid'' target are included in Figure~\ref{fig:pyrm}. In both cases, we see very similar results to the uniform case, with the other methods being clearly suboptimal compared to ours. We also point out that, for binary classification, the excess risks at the largest budget of \$500 are different by an order of magnitude---our method achieves an excess risk of about 0.02, while the Uniform and Nearest sampling plans incur an excess risk of more than 0.2! With this representing the probability of misclassifying an instance beyond the true linear classifier, this difference is quite meaningful.

\end{document}